\documentclass[twoside]{article}

\usepackage[preprint]{aistats2027}
\usepackage[utf8]{inputenc}
\usepackage[T1]{fontenc}   
\usepackage{url}           
\usepackage{booktabs}      
\usepackage{amsfonts}     
\usepackage{nicefrac}      
\usepackage{microtype}     
\usepackage{xcolor}       
\usepackage{bbm}

\usepackage[acronym,nowarn]{glossaries}
\glsdisablehyper
\makeglossaries
\usepackage{xspace}
\newacronym[firstplural=Bayesian neural networks]{BNN}{\textsc{bnn}}{Bayesian neural network}

\newcommand{\bnns}{\glspl{BNN}\xspace}
\newacronym[firstplural=universal conditional censity approximators]{UCDA}{\textsc{ucda}}{universal conditional density approximator}
\newcommand{\ucda}{\gls{UCDA}\xspace}
\newcommand{\ucdas}{\glspl{UCDA}\xspace}
\newacronym{EBSWD}{\textsc{ebswd}}{energy-based sliced Wasserstein distance}

\newacronym{SGHMC}{\textsc{sghmc}}{stochastic gradient Hamiltonian Monte Carlo}
\newcommand{\sghmc}{\gls{SGHMC}\xspace}
\newacronym[firstplural=Gaussian processes]{GP}{\textsc{gp}}{Gaussian process}
\newcommand{\gp}{\gls{GP}\xspace}

\newacronym{DWF}{\textsc{dwf}}{deep weight factorization}
\newcommand{\dwf}{\gls{DWF}\xspace}

\newacronym[firstplural=multi-layer perceptrons]{MLP}{\textsc{mlp}}{multi-layer perceptron}
\newcommand{\mlp}{\gls{MLP}\xspace}
\newcommand{\mlps}{\glspl{MLP}\xspace}
\newacronym[firstplural=convolutional neural networks]{CNN}{\textsc{CNN}}{convolutional neural network}

\newacronym{MCMC}{\textsc{mcmc}}{Markov chain Monte Carlo}

\newacronym{HMC}{\textsc{hmc}}{Hamiltonian Monte Carlo}

\newacronym{SWD}{\textsc{swd}}{sliced Wasserstein distance}

\newacronym{DSWD}{\textsc{dswd}}{distributional sliced Wasserstein distance}

\newacronym{RMSE}{\textsc{rmse}}{root mean squared error}
\newcommand{\rmse}{\gls{RMSE}\xspace}
\newacronym{NLL}{\textsc{nll}}{negative log-likelihood}
\newcommand{\nll}{\gls{NLL}\xspace}
\newacronym{ACC}{\textsc{acc}}{accuracy}
\newcommand{\acc}{\gls{ACC}\xspace}
\newacronym{MAP}{\textsc{map}}{maximum a posteriori}
\newcommand{\map}{\gls{MAP}\xspace}
\newacronym{MMD}{\textsc{mmd}}{maximum mean discrepancy}
\newcommand{\mmd}{\gls{MMD}\xspace}

\newacronym{RKHS}{\textsc{rkhs}}{reproducing kernel Hilbert space}

\newacronym{KL}{\textsc{kl}}{Kullback-Leibler}
\newcommand{\kl}{\gls{KL}\xspace}

\newacronym{ECE}{\textsc{ece}}{expected calibration error}
\newcommand{\ece}{\gls{ECE}\xspace}
\usepackage{caption}
\usepackage{listings}
\usepackage{subcaption}
\usepackage{float}

\usepackage{algorithm}
\usepackage{algorithmic}
\usepackage{float}
\usepackage{xurl}
\usepackage{amsmath}
\usepackage{amssymb}
\usepackage{mathtools}
\usepackage[round]{natbib}

\usepackage{amsthm}
\usepackage{thmtools}

\theoremstyle{plain}
\newtheorem{theorem}{Theorem}[section]
\newtheorem{proposition}[theorem]{Proposition}

\newtheorem{corollary}[theorem]{Corollary}

\theoremstyle{definition}
\newtheorem{definition}[theorem]{Definition}
\newtheorem{assumption}[theorem]{Assumption}

\theoremstyle{remark}
\newtheorem{remark}[theorem]{Remark}

\usepackage{hyperref}
\usepackage{cleveref}

\newcommand{\starref}[1]{\hyperref[#1]{$(\star)$}}

\begin{document}
\twocolumn[
\runningtitle{Sparsifying Stochasticity, Not Capacity}
\aistatstitle{Sparsifying Stochasticity, Not Capacity: Partial Stochasticity via Deep Weight Factorization of Prior Scales}
\runningauthor{Linhard, Filippone}
\aistatsauthor{Marius P. Linhard
  \And
  Maurizio Filippone}
\aistatsaddress{Université Grenoble Alpes, France \And KAUST, Saudi Arabia}]

\begin{abstract} 
Bayesian neural networks need not be fully stochastic to be universal conditional density approximators, but it remains open which parameters should be stochastic. We learn this split by applying deep weight factorization to the prior scales, which are the standard deviations of the parameter priors, while fitting the functional prior to a Gaussian process with a maximum mean discrepancy objective. A parameter whose prior scale falls below a cutoff becomes deterministic and is optimized during inference, so the regularizer sparsifies stochasticity rather than capacity. We give a certificate for universal conditional density approximation that is checkable in linear time, together with a minimal repair when it fails. We further show that the common hybrid scheme of sampling some parameters and optimizing the others is stochastic approximation for a type-II maximum a posteriori objective, and that coupled step sizes can leave a tracking error that does not vanish as the step size shrinks. On a bimodal target, the learned split stays close to an unconstrained reference across all budgets and is insensitive to the cutoff, while random masks that distribute the same prior scales across layers are worse by up to two orders of magnitude. On \texttt{UCI} benchmarks, our method performs on par with a fully stochastic network while keeping about half of its parameters deterministic.
\end{abstract}

\glsresetall
\section{Introduction}
As neural networks are increasingly deployed in real-world applications, reliable uncertainty quantification becomes essential. Bayesian neural networks address this by treating their parameters as random variables, which combines the flexibility of neural networks with principled uncertainty estimates.

A recent work \citep{sharma2023partiallystoch} has shown theoretically and empirically that it is not necessary to treat all parameters in a stochastic way to obtain good uncertainty estimates. It is even enough to make the number of stochastic parameters of the network match the target dimension as long as stochasticity is introduced in the early layers and the downstream network has general function approximation properties \citep{Leshno1993UAT}. This result is promising since it allows for the reduction in the number of stochastic parameters.

However, partial stochasticity complicates model selection by adding another modeling choice: which parameters should we treat in a stochastic way? The theory does not give any indication on how to make this choice. Due to the combinatorially large number of possible ways to assign parameters to the stochastic and deterministic groups, it is computationally impractical to find a solution via brute force.
An additional complication is how to choose sensible priors over stochastic parameters, which is crucial for obtaining good performance and sound uncertainty quantification.

Building on deep weight factorization \citep{kolb2025dwf}, we apply the factorization not to the network weights but to the prior standard deviations, which we refer to as {\em prior scales}. Sparsity in the scales means sparsity in stochasticity: a parameter whose prior scale collapses to zero is deterministic.
Unlike pruning and spike-and-slab-based constructions, which keep pruned parameters at zero in the posterior, our deterministic parameters remain free to be optimized during inference, so the capacity of the fitted network is retained.

Learning the split changes the inference problem. Partially stochastic networks are typically fitted by a hybrid scheme that samples the stochastic parameters while optimizing the deterministic ones. Despite its widespread use, it has not been established what this scheme approximates: it does not sample the joint posterior, nor is the deterministic block a standard \map estimate, since the stochastic parameters are integrated throughout optimization. With this work, we aim to close this gap.

Our contributions are: \textbf{(1)} a method that learns the deterministic/stochastic split by imposing deep weight factorization on prior scales. This sparsifies stochasticity rather than capacity, and for a factorization depth of $D\geq 3$ the non-zero scales are bounded away from zero at stationary points (\Cref{thm:no-small-vector}); \textbf{(2)} a functional prior fitted by the \mmd, where the degree-two polynomial kernel reduces to moment matching with weights fixed by the kernel, together with a decomposition of the \kl divergence that explains why a sparse network can match at most the moments (Proposition~\ref{prop:mmd}, Lemma~\ref{prop:pythagoras});
\textbf{(3)} a deterministic certificate for the universal conditional density approximation property, that is checkable in linear time, and a minimal repair when it fails (Proposition~\ref{prop:certificate}, Corollary~\ref{cor:MinRep}); \textbf{(4)}~a characterization of the hybrid inference scheme as stochastic approximation for a type-II \map objective, and a tracking bound showing that coupled sampler and optimizer step sizes leave an error term that does not shrink with the step size, which is attained by a Gaussian example (\Cref{thm:var}, 
\Cref{thm:track}, Corollary~\ref{cor:coupled}, Proposition~\ref{prop:BoundAttained}).
\section{Background} 
\paragraph{Notation} Suppose we have a dataset with $n$ input-target pairs $\mathcal{D} = \{\mathbf{x_i}, \mathbf{y_i}\}_{1 \leq i \leq n}$. We denote the parameters of a neural network by $\Theta$; the stochastic parameters are denoted by $\Theta_s$ and the deterministic ones by $\Theta_d$. By $\Theta^{(l)}$ we denote the parameters of layer $l$. Furthermore, we denote the neural network by $\mathbf{f}(\cdot; \Theta)$, the $l$'th layer of the neural network by $\mathbf{f}_l(\cdot; \Theta^{(l)})$, and the hyperparameters of the prior by $\Psi_\theta, \theta \in \Theta$ and $\Psi := \lbrace \Psi_\theta, \theta \in \Theta \rbrace$. In addition, we denote by $L$ the number of layers and by $d_i$ the dimension of the output of the $i$'th layer, where $d_0$ is the dimension of the input data. For a random variable $\mathbf Z$ we write $p(\mathbf Z)$ for its distribution, which need not have a density.

\paragraph{Functional prior}
The prior on the parameters also induces a functional prior through
\begin{equation*}
    p(\mathbf{f}; \Psi) = \int p(\mathbf{f}|\Theta) p(\Theta; \Psi)d\Theta
\end{equation*}
Here, $p(\mathbf{f}|\Theta)$ is defined by the network architecture and, therefore, deterministic. However, this now enables us to change our view from parameters to functions.
\paragraph{Inference} The posterior is intractable due to the intractability of the marginal likelihood $p(\mathcal{D}) = \int p(\mathcal{D}|\Theta) p(\Theta) d\Theta$. We sample it with \sghmc \citep{chen2014stochasticgradienthamiltonianmonte}, which scales to large datasets.
\paragraph{Partially and fully stochastic neural networks}
We first introduce fully-connected neural networks (\mlps), where the $l$'th layer is defined as 
\begin{equation*}
    \mathbf{f}_l(\mathbf{x}; \Theta^{(l)}) =  \mathbf{W}_l \varphi(\mathbf{f}_{l-1}(\mathbf{x}; \Theta^{(l-1)})) + \mathbf{b}_l 
\end{equation*}
where $\varphi$ is a possibly non-linear function (\textit{activation function}), $\mathbf{f}_{l-1}$ is the output of the $(l-1)$'th layer, $\mathbf{W}_l\in \mathbb{R}^{d_l \times d_{l-1}}$ is the weight matrix of the $l$'th layer, and $\mathbf{b}_l \in \mathbb{R}^{d_l}$ is the bias of the $l$'th layer. For the first layer we set $\mathbf{f}_1(\mathbf{x};\Theta^{(1)}) = \mathbf{W}_1\mathbf{x} + \mathbf{b}_1$.
In contrast to fully stochastic neural networks, where all parameters are stochastic \citep{neal1996bnn}, partially stochastic neural networks have both deterministic and stochastic parameters \citep{sharma2023partiallystoch}. 

\paragraph{Deep weight factorization} 
\citet{Hoff2017Lasso} showed that $\ell_1$ regularization is equivalent to $\ell_2$ regularization of a two-factor Hadamard parameterization, which avoids the non-differentiability of the $\ell_1$ penalty. \citet{kolb2025dwf} extended this to depth $D > 2$ through:
\begin{equation}
    \min_{\boldsymbol{\omega}_1, \dots, \boldsymbol{\omega}_D} \mathcal{L} \left(\mathbf{x},\mathbf{y};\bigodot_{i = 1}^D \boldsymbol{\omega}_i\right) + \frac{\lambda}{D} \sum_{i=1}^D \lVert \boldsymbol{\omega}_i \rVert_2^2,
\label{eq:dwf}
\end{equation}
where $\odot$ denotes the element-wise Hadamard-product.
As \Cref{eq:dwf} inflates the dimensionality of the objective, the optimization is very sensitive to the initialization of the parameters as reported in \citet{kolb2025dwf}. 
\section{Related Works} 
\paragraph{Functional priors in Bayesian neural networks} The specification of priors is an important, yet difficult task as the prior has a direct impact on the performance of the neural network and should capture any prior knowledge about the data. For a thorough review of priors in Bayesian deep learning we refer the reader to the work of \citet{fortuin2022priorsbayesiandeeplearning}. A common trick, which was also mentioned by
\citet{fortuin2022priorsbayesiandeeplearning}, is to change the perspective from the parameter space to the functional space and impose a prior on the function induced by the neural network \citep{flam2017mapping, Wild2022functioninference}. 
\citet{tran2022needgoodfunctionalprior} showed that optimizing the functional prior induced by the network to emulate a \gp before sampling from the posterior can improve performance substantially.

\paragraph{Partially stochastic neural networks}
Fully stochastic neural networks are well-known and extensively researched \citep{MacKay1995BNNs, neal1996bnn}. For a comprehensive review of \bnns we refer the reader to \citep{arbel2023primerbayesianneuralnetworks}. 

The notion of partially stochastic neural networks, however, is a rather recent development driven by the work of \citet{sharma2023partiallystoch}. In this work, the authors have contributed theoretically and empirically to the development and justification of partial stochasticity by proving that many popular architectures have the property of being \ucdas; this means that under mild topological constraints, these models can approximate any conditional distribution. 
Since then, interest in partially stochastic neural networks, their application, and their properties has grown (e.g., \citep{calvoordonez2024partiallystochasticinfinitelydeep, allec2025psbnnapplicmaterials, Prabhudesai2023psbnnapplicmed}). 

\paragraph{Sparsity in neural networks}
The weak \citep{frankle2018lottery} and strong lottery ticket hypotheses \citep{Natale2024SLTH} give the justification for efforts to sparsify neural networks. 
One possible way to impose sparsity on the parameters is to perform automatic relevance determination, which is placing hyperpriors on the priors and shrinking scales continuously \citep{neal1996bnn}. So a parameter is rarely exactly deterministic and the split is never explicit.
Many other works, e.g., \citep{hubin2024SBNN, Wang2021SparseMCMC, Deng2019SBDL}, have focused on spike-and-slab priors \citep{Mitchell1988SpikeAndSlab}.
However, all of these approaches prune capacity of the network by setting parameters to zero; our approach is different. We make the prior scales and mean learnable parameters and we then optimize the functional prior to prune the prior scales. This makes the parameters deterministic and the split explicit.

\paragraph{Selection of the stochastic and deterministic parameters}
The choice between stochastic and deterministic parameters for partially stochastic neural networks is a relatively new task for model selection. Closest to our setting, \citet{andrade2025detstochweights} select the stochastic parameters from a \map estimate. Subnetwork inference \citep{Daxberger21SNI}, which we include as SNI-PSNN, belongs to the same family of \map-based selection methods. These methods first train a full network and select the split from its \map estimate. Our split is learned from the functional prior, which uses only the inputs and requires no trained model. Sparse subspace variational inference \citep{li2024ssvi} removes parameters from the network and thereby reduces its capacity, whereas our deterministic parameters remain free to be optimized during inference. \citet{rochussen2024structured} fixes a structured subset of parameters in order to break weight symmetries, which is complementary to learning the split.

\section{Methods} 
\label{sec:methods}
Throughout the paper, we fit the functional prior by minimizing the \mmd between the functional prior of the network and the target \gp (for an introduction see \citep{rasmussen2008gps}) at random measurement locations.  
Previous work used the (sliced) Wasserstein distance for this purpose \citep{tran2022needgoodfunctionalprior,Tran2021BAE}. The \mmd \citep{gretton2012MMD} is cheap to estimate from samples and allows the kernel to be chosen for the task. We specify the kernels used in our experiments in \Cref{sec:exp}.
We propose a novel way to find a split between stochastic and deterministic parameters by employing a \dwf parameterization \citep{kolb2025dwf} for the prior scales of the parameters. 
In this section, we detail our method while supporting it with theoretical results. All proofs are in \Cref{app:theoRe}.

\subsection{Learning the split}
\label{sec:pri_opt}
\textbf{Parameterization:} For every $\theta \in \Theta^{(l)}$ we introduce factors $(\omega_\theta^{i})_{i=1}^D \in \mathbb{R}^D$, their product $u_\theta = \prod_{i=1}^D\omega_\theta^{(i)}$ and $\sigma_\theta := H(u_\theta)$, $H(u) = (h(u) - h(0))^2$, $h(u) = \log(1+\exp(u))$. We refer to $\sigma_\theta$ as the prior scale of $\theta$.
We chose this reparameterization of the prior scales to ensure that the standard deviations are non-negative with $H(u)= 0 \iff u=0$. The prior is $\theta \sim \mathcal{N}(\mu_\theta, \sigma_\theta^2)$ for each parameter $\theta \in \Theta$, where both $\boldsymbol \mu$ and $\boldsymbol \sigma$ are trainable parameters. We initialize $\boldsymbol \mu = \mathbf 0$. With an odd activation and a zero-mean target process, $\boldsymbol \mu = \mathbf 0$ is a stationary point of the objective, and the learned means remain at zero in our experiments. We perform the \dwf on the prior scales only and initialize them according to \citep{kolb2025dwf}.

For measurement locations $\mathbf X = (\mathbf x_1, \dots, \mathbf x_N)$, let $P^{\boldsymbol\mu,\boldsymbol\sigma}_{\mathbf X}$ denote the distribution of $(\mathbf f(\mathbf x_1;\Theta), \dots, \mathbf f(\mathbf x_N;\Theta))$ under the prior with means $\boldsymbol\mu$ and scales $\boldsymbol\sigma$, and $Q_{\mathbf X}$ the corresponding marginal of the \gp. For a kernel $k$ we define
\[
\mathcal L_k(\boldsymbol\mu, \boldsymbol\sigma) := \mathbb E_{\mathbf X \sim \mathcal{G}^N}\bigl[\mathrm{MMD}_k^2\bigl(P^{\boldsymbol\mu,\boldsymbol\sigma}_{\mathbf X}, Q_{\mathbf X}\bigr)\bigr],
\]
where the generator $\mathcal{G}$ of the measurement locations is described below. In practice both distributions are replaced by the empirical measures of function samples (Remark~\ref{rmk:mmd_optimize}). We cast the prior optimization problem as
\[
\min_{\boldsymbol\mu, \boldsymbol\omega} \mathcal L(\boldsymbol\mu, \boldsymbol\omega) := \mathcal L_k(\boldsymbol\mu, H(\mathbf u)) + \sum_{l=1}^L \frac{\lambda_l}{D} \sum_{\theta \in \Theta^{(l)}} \sum_{i=1}^D \bigl(\omega_\theta^{(i)}\bigr)^2 .
\]

During the optimization, a parameter whose prior scale falls below a cutoff $c$ is set to be deterministic, so that it receives no gradient from the objective. After optimization, we apply the minimal repair of Corollary~\ref{cor:MinRep} whenever the certificate of Proposition~\ref{prop:certificate} fails. On the \texttt{UCI} datasets this changes at most 34 parameters, or 0.3\% of the network.

A natural question is to what extent $c$ affects sparsity. We show in \Cref{thm:no-small-vector} that at stationary points there is a gap in absolute values between the values that are set to $0$ and the surviving prior scales. If $c$ lies within this gap, changing $c$ has no effect on the result.
\citet{kolb2025dwf} have shown that performing \dwf of depth $D$ is equivalent to optimizing with a $\ell_{\frac{2}{D}}$-quasi norm. Since $2/D < 1$ for $D \geq 3$, the penalty is only a quasi-norm, as the triangle inequality fails, and it is no longer convex. Its derivative grows without bound as $|u_\theta| \to 0$, so small non-zero values are expensive and are pushed to exactly zero, which brings the penalty closer to $\ell_0$ and makes it sparser than $\ell_1$ (cf. Remark~\ref{rmk:effectivePenalty}). For the statistical properties of $\ell_p$ penalties with $p < 1$ we refer the reader to \citet{knight2000asymptotics} and \citet{chartrand2007exact}, and for their realization through deep factorizations to \citet{kolb2025dwf}.
Proposition~\ref{prop:quasiNorm} extends this to layer-specific $\lambda_l$, which yields the following theorem:
\begin{restatable}[No small non-zero scales]{theorem}{PriorScales}
\label{thm:no-small-vector}
Let $D\geq3$, let $\theta\in\Theta^{(l)}$. Let $\tilde{\mathcal L}(\mathbf u) := \mathcal L_k(\boldsymbol \mu, H(\mathbf u))$ as a function of $u_\theta$ with all other variables fixed, and suppose $|\tilde{\mathcal{L}}'(u_\theta)|\leq G_l$. If $(\omega_\theta^{(1)}, \dots, \omega_\theta^{(D)})$ is a stationary point of $\mathcal{L}$ in these factors, then either all factors vanish, in which case $u_\theta = \sigma_\theta = 0$ or
\[
|u_\theta| \geq t_l := \left(\frac{2\lambda_l}{DG_l} \right)^{\frac{D}{D-2}} \quad \text{and} \quad \sigma_\theta \geq H(-t_l) >0
\]
\end{restatable}
The bound holds at exact stationary points. An optimizer step stops at approximate stationarity, and $t_l$ shrinks as the bound $G_l$ on the gradient grows, so the gap is not guaranteed in practice. In our experiments the surviving scales extend down to the cutoff, so the split depends on the choice of $c$ in practice. The fitted model, however, is insensitive to $c$ over more than an order of magnitude (\Cref{fig:toy_cutoff}).
\begin{remark}
\label{rmk:prior_opt}
    For comparable $G_l$, choosing $\lambda_1<\lambda_2<\dots<\lambda_L$  enforces $t_1<t_2<\dots<t_L$: later layers are sparsified more aggressively.
\end{remark}
A consequence of this theorem is that $c < \min_lH(-t_l)$ removes only scales that are exactly zero at every stationary point. 

\citet{borgwardt2006integrating} showed that an \mmd with a polynomial kernel of degree $d$ matches moments up to degree $d$. For the degree-two kernel used in our experiments, this takes the following form.
\begin{restatable}[The \mmd reduced to moment matching]{proposition}{MMDProp}
\label{prop:mmd}
Let $k(\mathbf{x},\mathbf{y})=(1+\langle \mathbf{x},\mathbf{y}\rangle)^2$ on $\mathbb{R}^N$, and let $\mathbf{M}_P:=\boldsymbol{\Sigma}_P+\boldsymbol{\mu}_P\boldsymbol{\mu}_P^\top$ denote the second-moment matrix.
Then for any two distributions with finite second moments,
\[
\mathrm{MMD}_2^2(P,Q)\;=\;2\,\|\boldsymbol{\mu}_P-\boldsymbol{\mu}_Q\|_2^2\;+\;\|\mathbf{M}_P-\mathbf{M}_Q\|_F^2 .
\]
Here the subscript $2$ of $\mathrm{MMD}_2^2(P,Q)$ denotes the use of the polynomial kernel of degree 2.
\end{restatable}
The weights of the mean and the second-moment term are thus fixed by the kernel and need not be tuned. In our implementation the objective is scaled by $\frac{1}{2N}$, which only rescales $\boldsymbol\lambda$.

To estimate the objective we draw measurement locations from a generator $\mathcal{G}$ which mixes the empirical distribution of the training inputs and a uniform distribution over the input range as suggested by \citet{sun2019functional} and also done by \citet{tran2022needgoodfunctionalprior}: 
$\mathcal{G} = \rho_M\, \hat{\mathbb{P}}_\mathbf{X_{\text{train}}} + (1-\rho_M)\,\mathcal{U}(\mathbf{x_{\min}},\mathbf{x_{\max}})$, i.e. each location is a training input with probability $\rho_M$ and uniform otherwise. We use $\rho_M = 0.7$ for the experiments. The choice of $\rho_M$ determines where the functional prior is actually constrained: the uniform component forces the match to hold away from the training data, which is precisely the region in which out-of-distribution uncertainty is evaluated.

\subsection{Preserving expressiveness}
\label{subsec:presEx}
We analyze how masking affects the \ucda property of \citet{sharma2023partiallystoch} (\Cref{theoremA}) through two notions.
\begin{definition}[Bias coverage]
\label{def:bias_coverage}
Let $S:= \{\theta \in \Theta^{(1)}: \sigma_\theta >0 \}$ be the surviving set in the first layer. The \emph{bias coverage} is $U(S) := \#\{ j \in \{1, \dots, d_1 \}: b^{(1)}_j \in S \}$, and $T := \{j: \, \text{no first-layer parameter of unit $j$ lies in $S$}\}$ is the set of fully deterministic units.
\end{definition}
\begin{remark}
The definition counts biases rather than weights for a reason. A stochastic input weight contributes $\sigma_{j,i}  \varepsilon_{j,i} x_i$, which vanishes at $\boldsymbol x = \boldsymbol 0$, so the scale $s_j$ in the proof of Proposition~\ref{prop:certificate} would not be bounded away from zero and the recovery map would fail there. A stochastic bias enters additively and keeps $s_j \geq \sigma^b_j >0$ for every input.
\end{remark}
\begin{definition}[\ucda of a masked network]
\label{def:mask_ucda}
Fix a mask. The masked network is a \ucda if for every $(\mathbf{X},\mathbf{Y})$ as in \Cref{theoremA} with a continuous generator and every $\varepsilon >0$, there exist values of the deterministic parameters and a Gaussian distribution with independent nondegenerate coordinates on $\Theta_s$ such that $\sup_{\mathbf{x} \in \mathcal{X}}d_{\mathrm{BL}}(p(\mathbf{f}(\mathbf{x};\Theta)), p(\mathbf{Y}|\mathbf{X}=\mathbf{x}))\leq \varepsilon$, where $d_{\mathrm{BL}}$ denotes the bounded-Lipschitz metric.
\end{definition}
The definition only states that the mask does not restrict expressiveness, not that the approximation is found in practice. It is hard to check directly, but the following certificate can be checked in linear time.
\begin{restatable}[Deterministic certificate]{proposition}{DetCert}
\label{prop:certificate}
Let $\mathcal{X} \subset \mathbb{R}^{d_0}$ be compact and let $\mathbf{f}$ have $L-1 \geq 2$ hidden layers with a continuous, strictly monotone, non-polynomial activation $\varphi$. Stochastic parameters may occur in any layer. If 
\[
U(S) \geq d_L \quad \text{and} \quad |T|\geq d_0,
\]
then the masked network is a \ucda. Both conditions are checkable in $\mathcal{O}(|\Theta^{(1)}|)$ time. 
\end{restatable}
\begin{remark}
If the deterministic first-layer weights are fixed, for example after training, the second condition becomes $\operatorname{rank}(\bar{\mathbf{W}}^{(1)}_T) = d_0$, where $\mathbf{\bar{W}}_T^{(1)}$ are the first-layer weights of the units in $T$, and Proposition~\ref{Prop:WeightRepair} gives the minimal repair for this case.
\end{remark}
Note that $\mathcal{X}$ need not be the raw input space. If a fixed continuous map precedes the network, the proposition applies unchanged with $d_0$ the dimension of its output.

If the certificate of Proposition~\ref{prop:certificate} fails, a minimal repair restores it by making at most $d_L$ first-layer biases stochastic and, if needed, the $d_0 - |T|$ cheapest units outside these $d_L$ fully deterministic (Corollary~\ref{cor:MinRep}).

\subsection{Inference}
The prior optimization provides three inputs to the inference stage. The split fixes the stochastic dimension $m := |\Theta_s|$ and hence determines whether the model can represent the target conditional at all, the surviving scales $\hat{\boldsymbol\sigma}_s$ define the prior over the stochastic block, and the means $\hat{\boldsymbol\mu}$ center the prior of both blocks. Let $(\hat{\boldsymbol\mu}, \hat{\boldsymbol\sigma})$ be the output of the prior optimization and set $k := |\Theta_d|$. We write $\ell(\boldsymbol\theta_d, \boldsymbol\theta_s) := \log p(\mathcal D \mid \boldsymbol\theta_d, \boldsymbol\theta_s)$, $Z(\boldsymbol\theta_d) := p(\mathcal D \mid \boldsymbol\theta_d) = \int e^{\ell(\boldsymbol\theta_d, \boldsymbol\theta_s)} p(\boldsymbol\theta_s)\,\mathrm d\boldsymbol\theta_s$ and $J_\tau(\boldsymbol\theta_d) := \log Z(\boldsymbol\theta_d) - \frac{1}{2\tau^2}\|\boldsymbol\theta_d - \hat{\boldsymbol\mu}_d\|_2^2$, where $\tau = \gamma \log(2)$ is governed by a hyperparameter $\gamma$. We use this reparameterization so that $\gamma = 1$ corresponds to the value of the softplus function at $0$. The hyperparameter $\gamma$ controls how far the deterministic parameters
may move from their prior mean. By Proposition~\ref{prop:gamma}, $\gamma \downarrow 0$
freezes them at $\hat{\boldsymbol\mu}_d$, a finite $\gamma$ gives a type-II \map estimate, and $\gamma \uparrow \infty$ gives type-II maximum
likelihood. 

The deterministic block is optimized with \texttt{Adam} \citep{Kingma2014Adam} under the prior $\mathcal N(\hat{\boldsymbol\mu}_d, \tau^2 \mathbf I_k)$, and the stochastic block is sampled with \sghmc \citep{chen2014stochasticgradienthamiltonianmonte} under the prior $\prod_{\theta \in \Theta_s} \mathcal N(\hat\mu_\theta, \hat\sigma_\theta^2)$. Let $\Pi_\tau$ denote the joint posterior of $(\boldsymbol\theta_d, \boldsymbol\theta_s)$ under these two priors.

\begin{restatable}[The near-degenerate approximation is a type-II \map estimate]{theorem}{InfScheme}
\label{thm:var}
Let $\mathcal{P}_G$ be the set of densities $q_s$, with finite entropy and $\mathbb{E}_{q_s}[G+|\log p(\boldsymbol \theta_s)| + |\ell (\boldsymbol 0,\boldsymbol \theta_s)|] < \infty$, where $G$ is the envelope from Assumption~\ref{ass:reg} (ii), and let $q_\varepsilon := \mathcal{N}(\boldsymbol \theta_d, \varepsilon^2\boldsymbol I_k)$. Under Assumption~\ref{ass:reg}, for all $\boldsymbol \theta_d \in \mathbb{R}^k$ and $q_s \in \mathcal{P}_G$,
\begin{align*}
\mathrm{KL}(q_\varepsilon \otimes q_s \| \Pi_\tau) &= -\frac{k}{2} \log (2\pi e \varepsilon^2) + F(\boldsymbol \theta_d, q_s) + R_\varepsilon \\
F(\boldsymbol \theta_d, q_s) &= \mathrm{KL}(q_s\| p(\cdot| \mathcal{D}, \boldsymbol \theta_d)) - J_\tau(\boldsymbol \theta_d) + c,
\end{align*}
with $c \in \mathbb{R}$ being a constant, $|R_\varepsilon| \leq C(\boldsymbol \theta_d) \mathbb{E}_{ q_s}[G]\varepsilon + \frac{k \varepsilon^2}{2 \tau^2}$ for $\varepsilon \leq 1$ and $C(\boldsymbol \theta_d)$ locally bounded. The minimizers of $F$ are $q^\star_s = p(\cdot | \mathcal{D}, \boldsymbol \theta^\star_d)$ with $\boldsymbol \theta^\star_d \in \arg \max J_\tau$.
\end{restatable}

By Fisher's identity, $\nabla_{\boldsymbol\theta_d} \log Z(\boldsymbol\theta_d) = \mathbb E_{p(\boldsymbol\theta_s \mid \mathcal D, \boldsymbol\theta_d)}[\nabla_{\boldsymbol\theta_d} \ell(\boldsymbol\theta_d, \boldsymbol\theta_s)]$ \citep{Capp2005InferenceIH}, so the hybrid scheme is stochastic approximation EM for $J_\tau$, with the sampler acting as a Monte Carlo E-step. Because the target of the sampler moves with $\boldsymbol\theta_d$, the step sizes matter. The next result shows that if sampler and optimizer share a step size, the bound on the tracking error does not vanish as the step size shrinks.
\begin{restatable}[Coupled step sizes yield a non-vanishing tracking error]{corollary}{CoupledStepSize}
\label{cor:coupled}
Assume in addition to the setting of \Cref{thm:track} that $\rho (h) = 1-ch + o (h)$ with $c>0$ and $\delta_t \leq G\eta$ for all $t$. Then for all $t$
\[
e_t \leq \rho (h)^te_0 + \frac{L_\pi G\eta}{ch + o(h)}
\]
(a) Coupled. $h = \eta$. The stationarity part of the bound tends to $\frac{L_\pi G}{c}$ as $\eta \to 0$.
(b) Decoupled. If $h \to 0$ and $\frac{\eta}{h} \to 0$, then $\sup_t[W_1(\boldsymbol \mu_t^{\boldsymbol \theta_s}, \pi_{\boldsymbol \theta_{d_t}})  - \rho(h)^te_0] \to 0$
\end{restatable}
This bound is attained for a Gaussian example (Proposition~\ref{prop:BoundAttained}) and therefore sharp. We therefore let the optimizer step size decay faster than the sampler step size (\Cref{App:experimental_details}).
\section{Experiments} 
\label{sec:exp}
Unless stated otherwise, we use a \mlp with two hidden layers of 100 units and a $\tanh$ activation. For regression we use \texttt{CONCRETE}, \texttt{WINE}, \texttt{BIKESHARING} and \texttt{BOSTON}, and for classification \texttt{BANKNOTE}, \texttt{HTRU2}, \texttt{ONLINESHOPPERS} and \texttt{CREDIT}. All of these datasets come from the \texttt{UCI} suite \citep{kelly2023uci}. For each dataset we use ten random splits into 90\% training and 10\% test data and report the mean and standard deviation over the splits. We evaluate regression by \rmse, \nll and \ece, and classification by \acc, \nll and \ece. Details are given in \Cref{App:experimental_details}. For the prior optimization we use the \mmd with the degree-two polynomial kernel $k(\mathbf x, \mathbf y) = (1 + \langle \mathbf x, \mathbf y \rangle)^2$, which matches the mean and the second moment of the function values (Proposition~\ref{prop:mmd}).
\subsection{Learned Split}
\begin{figure}[ht]
    \centering
    \includegraphics[width=\linewidth]{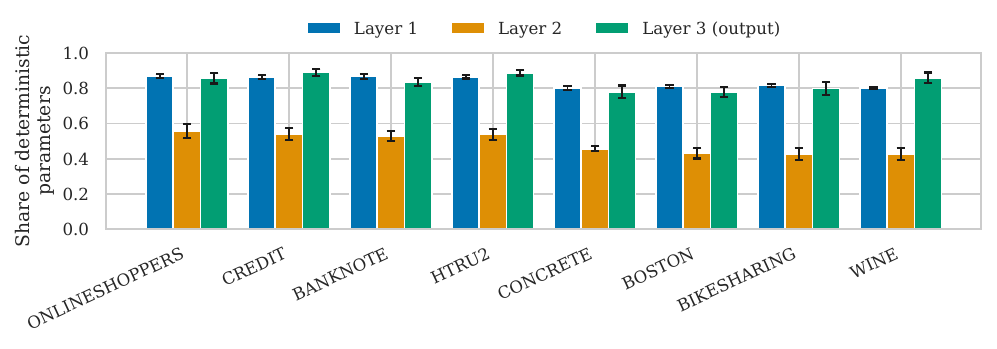}
    \caption{Share of deterministic parameters per layer for DWF-PSNN before repair, with mean and standard deviation over ten splits}
    \label{fig:learned_split}
\end{figure}
\Cref{fig:learned_split} shows the learned split. It is stable across data splits, keeps the first and the output layer largely deterministic and leaves about half of the middle layer stochastic. The coverage condition holds in every run. The condition $|T| \ge d_0$ fails on datasets with many input features, and the minimal repair then changes between 1 and 34 parameters (\Cref{App:LearnedSplit}).

\subsection{Conditions of the certificate}
\begin{figure}[ht]
\centering
\includegraphics[width=\linewidth]{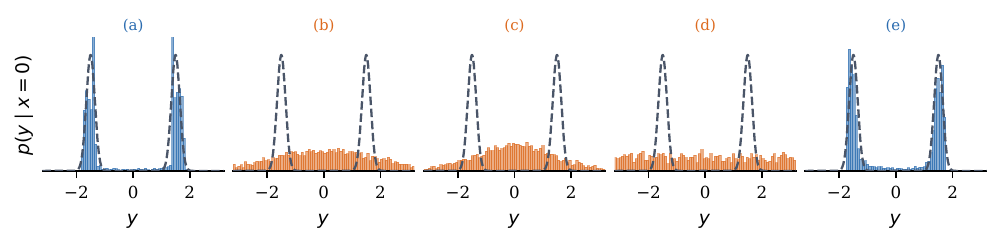}
\caption{Predictive distribution at the bimodal input, with the true conditional dashed and certified configurations in blue. (a) Unconstrained reference, (b) output layer randomized, (c) no nonlinearity, (d) every weight randomized with fixed scales, (e) one randomized bias.}
\label{fig:ucda_toy}
\end{figure}
The certificate constrains a masked network through bias coverage, through the number of fully deterministic units, and through the activation. \Cref{fig:ucda_toy} shows one configuration per condition on a target whose conditional distribution is bimodal, so that any Gaussian conditional is excluded by construction. In the reference (a) the noise enters as an additional input and every weight is deterministic, so the map is unconstrained and the reachable error is limited only by optimization. A single randomized bias in an admitted position (e) comes close to that reference. Randomizing the output layer (b) leaves the bias coverage at zero and fails, with sixty-four randomized parameters as much as with one. Randomizing every weight with fixed prior scales (d) fails from the opposite side, since no fully deterministic units remain. Removing the nonlinearity (c) fails while count and position are correct. The three conditions are therefore not vacuous and are needed together. \Cref{tab:ucda_main} numerically supports the findings by providing the mean Wasserstein-1 distance. 

\begin{table}[ht]
\centering
\small
\caption{Mean Wasserstein-1
distance to the true conditional over nine held out inputs, with the
standard error over five data seeds. The sampling floor is the distance
between two independent samples from the true conditional at this sample
size, and the point predictor is what a
model reaches by matching the conditional median.}
\label{tab:ucda_main}
\begin{tabular}{ccc}
\toprule
variant & condition violated & $W_1$ \\
\midrule
sampling floor  & --- & .020  \\
point predictor & --- & .681  \\
\midrule
noise input, $d_L$   & reference, none &  $.053 \pm .004$  \\
layer 1, $d_L$   & none       & $.074 \pm .002$ \\
\midrule
layer 1, $d_L$, weight & vanishes at $x=0$   & $.284 \pm .002$  \\
output layer, $d_L$        & bias coverage         & $.568 \pm .002$ \\
output layer, all        & bias coverage        & $.495 \pm .004$ \\
layer 1, $d_L$, linear & activation            & $.935 \pm .002$  \\
full                     & deterministic units              & $5.411 \pm .008$ \\
\midrule
deterministic & no randomness & $.686 \pm .002$ \\
\bottomrule
\end{tabular}
\end{table}
We have added a more detailed version of \Cref{tab:ucda_main} in \Cref{app:ucda} (\Cref{tab:ucda}). There we change the notation from layer 1 to first layer.

\begin{figure}[ht]
\centering
\includegraphics[width=\linewidth]{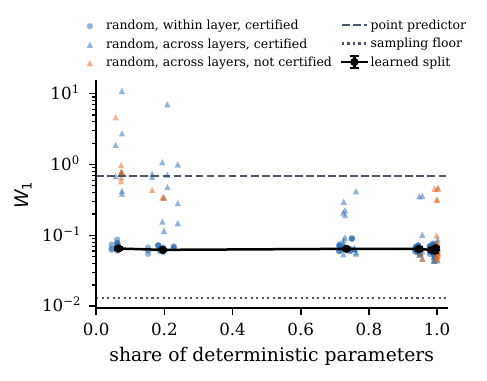}
\caption{Learned and random splits on the bimodal target over five data seeds. Random masks permute the learned scales within each layer (circles) or across all layers (triangles) and are colored by whether they satisfy the certificate.}
\label{fig:toy_budget_certificate}
\end{figure}
\begin{figure}[ht]
\centering
\includegraphics[width=\linewidth]{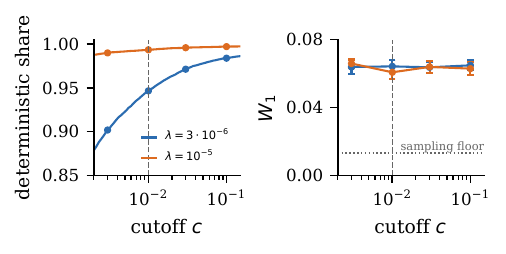}
\caption{Cutoff applied after the optimization on the bimodal target, over five data seeds. Left, share of deterministic parameters. Right, distance to the true conditional after refitting with the thresholded scales.}
\label{fig:toy_cutoff}
\end{figure}
Randomizing up to sixty-four parameters of the output layer never approaches the certified configuration (\Cref{fig:ucda-count} in \Cref{app:ucda}), so the count of stochastic parameters does not substitute for their position.
\Cref{fig:toy_budget_certificate} varies the budget instead. The learned split stays at $W_1 \approx 0.06$ between 6.5\% and 99.6\% deterministic parameters, and permuting its scales within a layer has no effect, since the units of a layer are exchangeable. Permuting them across layers is worse by up to two orders of magnitude and often worse than the fully deterministic model. The learned split thus matters through how it allocates stochasticity across layers. The certificate is sufficient but not necessary and leaves the scales free, so it does not separate these masks with fixed scales cleanly. 
\Cref{fig:toy_cutoff} varies the cutoff applied after the optimization. The prior scales do not separate into two groups, so the share of deterministic parameters grows with $c$, from 90.2\% at $c = 0.003$ for $\lambda = 3 \cdot 10^{-6}$ to 99.7\% at $c = 0.1$ for $\lambda = 1 \cdot 10^{-5}$. The distance to the true conditional nevertheless stays between 0.061 and 0.066 for all cutoffs and both values of $\lambda$, so the scales removed by a larger cutoff are too small to affect the fitted model.

\subsection{The role of $\gamma$}
We study the regimes of Proposition~\ref{prop:gamma} on \texttt{CONCRETE}.
\begin{table}[ht]
\centering
\caption{Results on \texttt{CONCRETE} for regimes (ii) and (iii) (mean $\pm$ std), with the drift $\delta := \|\boldsymbol\theta_d^{(T)} - \boldsymbol{\hat\mu}_d\|$ of the deterministic parameters. The sampling hyperparameters differ from those of the benchmark (\Cref{App:experimental_details}).}
\label{tab:regimes}
\small
\setlength{\tabcolsep}{4pt}
\begin{tabular}{lccc}
\toprule
$\gamma$ & $0.03$ & $2.5$ & $10^{13}$ \\
\midrule
\rmse ($\downarrow$) & .316 $\pm$ .033 & .296 $\pm$ .029 & \textbf{.291 $\pm$ .027} \\
\nll ($\downarrow$)  & .274 $\pm$ .095 & .205 $\pm$ .074 & \textbf{.193 $\pm$ .066} \\
\ece ($\downarrow$)  & \textbf{.085 $\pm$ .059} & .115 $\pm$ .083 & .119 $\pm$ .074 \\
$\delta$             & .04 $\pm$ .01 & 23.97 $\pm$ .48 & 76.94 $\pm$ 4.41 \\
\bottomrule
\end{tabular}
\end{table}
Freezing the deterministic parameters at $\boldsymbol{\hat\mu}_d = \mathbf 0$ in regime (i) gives an \rmse of $.997 \pm .001$ and an \nll of $4.74 \pm .02$, since all deterministic parameters are then zero. \Cref{tab:regimes} shows regimes (ii) and (iii). As $\gamma$ grows, the deterministic parameters move further from their prior mean, and \rmse and \nll improve monotonically. For $\gamma = 0.03$ the drift stays small, in line with Corollary~\ref{cor:DriftRadius}. The \ece behaves differently. It is lowest for $\gamma = 0.03$ and already close to its value in regime (iii) for $\gamma = 2.5$, although all differences in \ece lie within one standard deviation. Small values of $\gamma$ thus trade a little predictive accuracy for better calibration, while every finite $\gamma$ improves substantially on regime (i).
\subsection{Prior optimization}
\begin{figure}[ht]
    \centering
    \includegraphics[width=\linewidth]{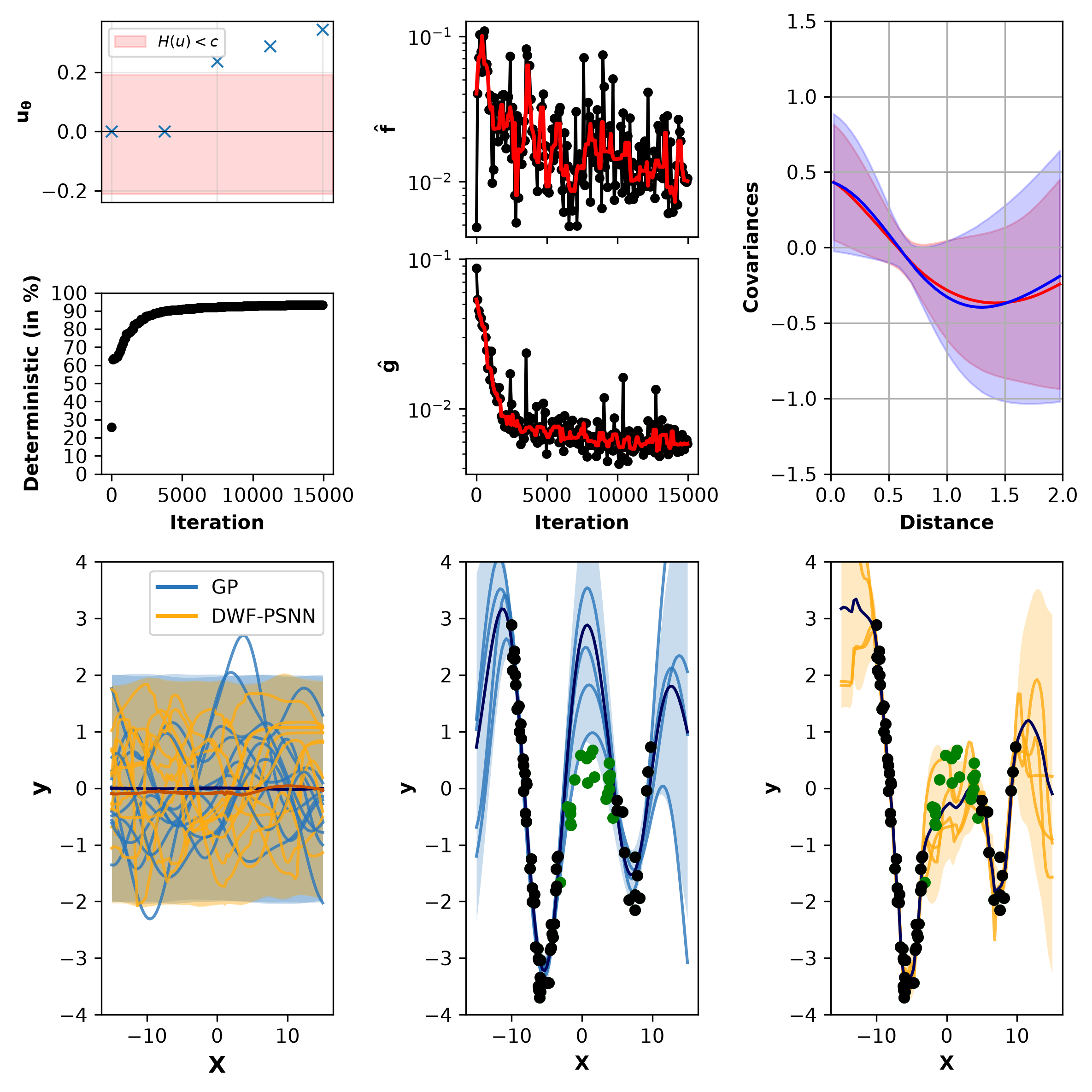}
    \caption{Top (from left to right): the deterministic fraction against iterations and the smallest surviving and the largest pruned prior scales; the mean ($\hat{f})$ and the uncentered second-moment $(\hat{g})$ terms on a logarithmic scale;  the covariogram of the optimized prior against the target \gp. Bottom (from left to right): DWF-PSNN prior and \gp sample paths ; the \gp posterior; the DWF-PSNN posterior}
    \label{fig:prior_opt}
\end{figure}
Here, we run the prior optimization on a 1D toy dataset. The data is drawn from a \gp and we artificially introduce a gap in the middle.
\Cref{fig:prior_opt} shows that the objective and both moment terms converge. The sample paths of the prior and the \gp still differ in smoothness, which reflects the non-Gaussian part that the objective does not target (cf. Lemma~\ref{prop:pythagoras}). The posterior nevertheless gives sensible uncertainty away from the data.

\subsection{\texttt{UCI} benchmarks}
We compare our approach (\textbf{DWF-PSNN}) to a fully stochastic network, where the functional prior is optimized (\textbf{ALS-BNN}); a partially stochastic network, where the first layer is stochastic, the subsequent layers are deterministic, and the functional prior is optimized (\textbf{FLS-BNN}); a partially stochastic neural network that makes the same overall share of parameters deterministic as DWF-PSNN, chosen uniformly at random across all layers and without repair (\textbf{RAND-PSNN}); the partially stochastic neural network as introduced in \citet{Daxberger21SNI} (\textbf{SNI-PSNN}), which represents \map-based selection of the stochastic parameters. 
On these benchmarks, whose likelihoods are Gaussian or categorical, RAND-PSNN performs on par with DWF-PSNN, so the allocation of stochasticity matters less here than on the bimodal target of \Cref{fig:toy_budget_certificate}. We refer the reader to \Cref{App:AddRes} for the detailed results and \Cref{App:experimental_details} for more information on the experimental details. \Cref{tab:mean_ranks} shows that DWF-PSNN ranks at least as well as the fully stochastic ALS-BNN in all six combinations, although it keeps about half of its parameters deterministic.
\begin{table}[ht]
\centering
\caption{Mean ranks across the classification and regression datasets, computed from the mean performance on each dataset with tied ranks averaged. Lower ranks indicate better performance. The Nemenyi critical difference (CD) at $\alpha=0.05$ is $3.05$ for both tasks ($N=4$, $k=5$). The only pairwise difference that exceeds it is between RAND-PSNN and SNI-PSNN in classification \nll.}
\label{tab:mean_ranks}
\small
\setlength{\tabcolsep}{6pt}
\begin{tabular}{lccc}
\toprule
\multicolumn{4}{c}{\textbf{Classification ($N=4$)}} \\
\midrule
Method & \acc ($\uparrow$) & \nll ($\downarrow$) & \ece ($\downarrow$) \\
\midrule
DWF-PSNN  & \textbf{2.38} & 2.63 & 2.88 \\
ALS-BNN   & 2.63 & 2.63 & 3.13 \\
FLS-BNN   & 2.88 & 2.88 & \textbf{2.50} \\
RAND-PSNN & \textbf{2.38} & \textbf{1.88} & 2.63 \\
SNI-PSNN  & 4.75 & 5.00 & 3.88 \\
\midrule
Friedman $p$ & 0.092 & 0.058 & 0.736 \\
\addlinespace[4pt]
\midrule
\multicolumn{4}{c}{\textbf{Regression ($N=4$)}} \\
\midrule
Method & \rmse ($\downarrow$) & \nll ($\downarrow$) & \ece ($\downarrow$) \\
\midrule
DWF-PSNN  & \textbf{2.00} & \textbf{2.25} & 2.50 \\
ALS-BNN   & 3.25 & \textbf{2.25} & 2.75 \\
FLS-BNN   & 3.38 & 3.50 & 3.50 \\
RAND-PSNN & 2.63 & 3.00 & \textbf{2.25} \\
SNI-PSNN  & 3.75 & 4.00 & 4.00 \\
\midrule
Friedman $p$ & 0.543 & 0.434 & 0.493 \\
\bottomrule
\end{tabular}
\end{table}

\begin{table}[ht]
\centering
\caption{Certificate on the UCI datasets. Share of splits in which the learned mask satisfies the certificate before repair, mean number of parameters changed by the repair, and share of 2{,}000 random masks with the learned budget that satisfy the certificate.}
\label{tab:certificate_uci}
\small
\setlength{\tabcolsep}{4pt}
\begin{tabular}{lcccc}
\toprule
Dataset & $d_0$ & learned & repair & random \\
\midrule
\texttt{BANKNOTE}       & 4  & 1.00 & 0.0  & 1.00 \\
\texttt{HTRU2}          & 8  & 1.00 & 0.0  & 1.00 \\
\texttt{CONCRETE}       & 8  & 0.90 & 0.1  & 0.97 \\
\texttt{WINE}           & 10 & 0.00 & 5.5  & 0.09 \\
\texttt{BIKESHARING}    & 12 & 0.00 & 7.3  & 0.01 \\
\texttt{BOSTON}         & 13 & 0.00 & 7.8  & 0.00 \\
\texttt{ONLINESHOPPERS} & 17 & 0.00 & 11.7 & 0.00 \\
\texttt{CREDIT}         & 23 & 0.00 & 28.1 & 0.00 \\
\bottomrule
\end{tabular}
\end{table}

\Cref{tab:certificate_uci} shows that the certificate fails on the datasets with many input features, both for the learned masks before repair and for random masks with the same budget. The repair changes on average at most 28.1 parameters per split and never more than 34.
\section{Conclusion}
Learning with neural networks in a partially stochastic framework requires determining whether each parameter should be treated as deterministic or stochastic. 
These choices affect guarantees of universal approximation and performance.  
In this paper, we proposed a novel approach to learn such a split by applying deep weight factorization to the prior scales, while matching the functional prior to a Gaussian process with an \mmd objective. 
A certificate that is cheap to check guarantees universal conditional density approximation, and a minimal repair restores it when it fails. We characterized the inference scheme as a stochastic approximation for a type-II \map objective. Then we showed that  coupled step sizes of the sampler and the optimizer can leave a persistent tracking error. Our experiments show that the learned split allocates stochasticity where it matters and that the method is competitive on \texttt{UCI} benchmarks. 

\subsection{Limitations and Future Work}

\paragraph{Beyond \mlps.} By the observation after Proposition~\ref{prop:certificate}, the certificate covers the head of a convolutional or transformer model with a fixed stack below it. Applying the split there is a natural next step, but Corollary~\ref{cor:MinRep} needs a head wider than the feature dimension, which most architectures do not satisfy. The stack itself, with its shared weights, stays outside our analysis.

\paragraph{Closed-loop dynamics.} \Cref{thm:track} and Corollary~\ref{cor:coupled} assume that the optimizer path does not depend on the sampler noise. In our algorithm the \texttt{Adam} updates use the current samples, so the analysis covers the tracking part of the error but not the feedback between the two blocks. Extending the bounds to this closed-loop setting is left for future work.

\subsection{AI Use Statement}
In this work, we used generative AI tools to help formulate mathematical claims, to provide ingredients for the proofs and assist in writing them, to give feedback on the methodology and the experiments, and to help implement the methods. Additionally, we used generative AI tools to edit the text for readability and to write and edit code. We have reviewed all AI-assisted work. All proofs were checked line by line by the authors, the code was tested for correctness, and the text was checked for plagiarism. We take responsibility for the final content of this work, including text, claims, or artifacts produced with the aid of generative AI.

\bibliography{linhard}
\bibliographystyle{apalike}
\newpage
\onecolumn
\appendix
\section{Theoretical Results and Proofs}
\label{app:theoRe}
\subsection{Learning the split}
\label{app:theoReSplit}
\begin{proposition}[Layer-weighted quasi-norm]
\label{prop:quasiNorm}
Since $\mathcal{L}_k$ depends on the factors only through $\mathbf u$,
\[
\min_{\boldsymbol{\omega}} \mathcal{L}(\boldsymbol{\mu}, \boldsymbol{\omega}) = \min_{\mathbf{u}} \mathcal{L}_k (\boldsymbol{\mu}, H(\mathbf{u})) + \sum_{l=1}^L\lambda_l \sum_{\theta \in \Theta^{(l)}} |u_\theta|^{\frac{2}{D}}
\]
\end{proposition}

\begin{proof}
Fix $\theta$ and $u_\theta = \omega$. By the AM-GM inequality $\frac{1}{D}\sum_i (\omega^{(i)})^2 \geq (\Pi_i(\omega^{(i)})^2)^{\frac{1}{D}} = |\omega|^{\frac{2}{D}}$, with equality if and only if $|\omega^{(i)}| = |\omega|^{\frac{1}{D}}$ for all $i$. Hence, the inner minimum of $\sum_i(\omega^{(i)})^2$ under the constraint $\prod_i \omega^{(i)} = \omega$ equals $D|\omega|^{\frac{2}{D}}$. Multiplying by $\frac{\lambda_l}{D}$ and summing over $\theta$ and $l$ gives the claim.
\end{proof}

\begin{remark}[Effective penalty of the prior scales]
\label{rmk:effectivePenalty}
Since for our reparameterization we have $\sigma_\theta \approx \frac{u_\theta^2}{4}$ near zero, the penalty $\lambda_l |u_\theta|^{\frac{2}{D}}$ behaves like $\lambda_l (4\sigma_\theta)^{\frac{1}{D}}$. On the prior scale itself the regularizer is therefore of $\ell_{\frac{1}{D}}$ type, which is sparser than the $\ell_{\frac{2}{D}}$ quasi-norm acting on $\boldsymbol u$.
\end{remark}

\PriorScales*
\begin{proof}
Write $v := u_\theta$ and $g:= |\tilde{\mathcal{L}}'(v)|$. Stationarity in the $i$-th factor reads $\tilde{\mathcal{L}}'(v)\Pi_{j\neq i}v^{(j)} + \frac{2\lambda_l}{D}v^{(i)} = 0$. Multiplying by $v^{(i)}$ gives $\tilde{\mathcal{L}}'(v)v + \frac{2\lambda_l}{D}(v^{(i)})^2 = 0$, so
\[
(v^{(i)})^2 = -\frac{D}{2\lambda_l}\tilde{\mathcal{L}}'(v)v =: s^2 \quad \text{for every $i$,} \quad |v| = s^D.
\]
If $s=0$, all factors vanish. Otherwise $s>0$, and $s^2 = \frac{D}{2\lambda_l}g|v| = \frac{Dg}{2\lambda_l}s^D$ forces $g>0$ and $s^{D-2} = \frac{2\lambda_l}{Dg}$. Hence, $|v| = s^D = (\frac{2\lambda_l}{Dg})^{\frac{D}{D-2}}$. This expression is decreasing in $g$, so $g \leq G_l$ yields $|v| \geq t_l$. The bound on $\sigma_\theta$ follows from $\min_{|u|\geq t} H(u) = H(-t)$.
\end{proof}

The following Lemma is a special case of the Pythagorean identity onto exponential families \citep{Csisizar1975KL}:
\begin{restatable}[Pythagorean decomposition]{lemma}{PythIdentity}
\label{prop:pythagoras}
Let $P$ be a distribution on $\mathbb{R}^N$ with finite second moments, mean $\boldsymbol \mu$
and positive definite covariance matrix $\boldsymbol \Sigma$, and let $\tilde P := \mathcal{N}(\boldsymbol \mu,\boldsymbol \Sigma)$. Then for every Gaussian $Q = \mathcal{N}(\boldsymbol \mu_Q, \boldsymbol \Sigma_Q)$ with positive definite covariance matrix,
\[
\mathrm{KL}( P\|Q) \;=\; \underbrace{\mathrm{KL}( P\|\tilde P)}_{\text{non-Gaussianity of } P}
\;+\; \underbrace{\mathrm{KL}(\tilde P\|Q)}_{\text{moment discrepancy}} .
\]
In particular, if $P$ has no Lebesgue density, both $\mathrm{KL}(P\|\tilde P)$ and $\mathrm{KL}(P\| Q)$ are infinite.
\end{restatable}

\begin{proof}
$\tilde P$ and $Q$ are equivalent to Lebesgue measure. If $P$ is not absolutely continuous, then $\mathrm{KL}(P\|Q) = \mathrm{KL}(P\|\tilde{P}) = \infty$ and the identity holds trivially. Otherwise let $p$ be the density of $P$, $q$ the density of $Q$, and $\tilde{p}$ be the density of $\tilde{P}$. Since $\log q$ and $\log \tilde{p}$ are quadratic polynomials and $P, \tilde{P}$ share their first two moments, $\mathbb{E}_P[\log q] = \mathbb{E}_{\tilde{P}}[\log q]$ and also $\mathbb{E}_P[\log \tilde{p}]=\mathbb{E}_{\tilde{P}}[\log \tilde{p}]$, and all four quantities are finite. The integral $\int p \log p$ lies in $(-\infty, \infty]$ because the entropy of $P$ is bounded by that of $\tilde{P}$. If it equals $\infty$, both sides are infinite. Otherwise
\[
\mathrm{KL}(P\|\tilde{P}) + \mathrm{KL}(\tilde{P}\|Q) = \int p \log p - \mathbb{E}_P[\log \tilde{p}] + \mathbb{E}_{\tilde{P}}[\log \tilde{p}] - \mathbb{E}_{\tilde{P}}[\log q] = \int p \log p - \mathbb{E}_P[\log q] = \mathrm{KL}(P\|Q)
\]
\end{proof}

\begin{remark}[Why moments]
\label{rem:moments}
When the number of stochastic parameters drops below the number of measurement points, the network output has no Lebesgue density. By Lemma~\ref{prop:pythagoras} both divergences are then infinite, and only the moment discrepancy can be controlled. This motivates an objective based on the first two moments. By Proposition~\ref{prop:mmd}, $\mathrm{MMD}_2^2$ vanishes
exactly when the first two moments agree, so $\mathcal L_2$ is such an
objective. For deeper networks the non-Gaussianity term can be substantial
\citep{vladimirova2019understanding, yaida2020non}, and a
characteristic kernel could target it.
\end{remark}

\begin{remark}[Hierarchical targets]
\label{rem:hierarchical}
When the kernel hyperparameters of the target process carry hyperpriors, as in our experiments, the target $Q$ is a mixture of Gaussians and not itself Gaussian. Proposition~\ref{prop:mmd} holds for arbitrary distributions with finite second moments and is unaffected. Lemma~\ref{prop:pythagoras} applies with $Q$ replaced by its Gaussian moment projection $\tilde Q = \mathcal N(\boldsymbol \mu_Q, \boldsymbol \Sigma_Q)$, which is exactly the Gaussian whose moments our objective matches.
\end{remark}

\MMDProp*

\begin{proof}
Expanding, $k(\mathbf{x},\mathbf{y}) = 1 + 2 \langle \mathbf{x},\mathbf{y} \rangle + \langle \mathbf{ x},\mathbf{ y} \rangle^2$ and $\langle \mathbf{x}, \mathbf{y} \rangle^2 = \langle \mathbf{x}\mathbf{ x}^\top,\mathbf{y}\mathbf{y}^\top \rangle_F$. Hence, $k(\mathbf{x},\mathbf{y}) = \langle \varphi(\mathbf{x}), \varphi(\mathbf{y}) \rangle$ with the feature map $\varphi(\mathbf{x}) = (1, \sqrt{2}\mathbf{x}, \mathbf{x}\mathbf{x}^\top) \in \mathbb{R} \times \mathbb{R}^N \times \mathbb{R}^{N \times N}$, where the last component carries the Frobenius inner product. Since $\mathbb{E}_P[\|\varphi(\mathbf{X})\|] = \mathbb{E}_P[\sqrt{k(\mathbf{X},\mathbf{X})}] = \mathbb{E}_P[1 + \|\mathbf{X}\|^2] < \infty$, the mean embedding $m_P = \mathbb{E}_P[\varphi(\mathbf{X})] = (1, \sqrt{2}\boldsymbol{\mu}_P, \mathbf{M}_P)$ exists, and 
\[
\mathrm{MMD}^2_2 (P,Q) = \| \mathbf{m}_P - \mathbf{m}_Q \|^2 = 2 \| \boldsymbol{\mu}_P - \boldsymbol{\mu}_Q \|_2^2 + \| \mathbf{M}_P - \mathbf{M}_Q \|^2_F
\]
\end{proof}

\begin{remark}[What is optimized in practice]
\label{rmk:mmd_optimize}
$N$ is the number of measurement locations. In every step we draw
$\mathbf X$ from $\mathcal G^N$ and compute the empirical moments from
function samples. The loss of this step is
$\frac{\mathrm{MMD}_2^2(\hat P_{\mathbf X}, \hat Q_{\mathbf X})}{2N}$ for the
empirical measures $\hat P_{\mathbf X}$ and $\hat Q_{\mathbf X}$, which is
the biased \mmd estimator scaled by $\frac{1}{2N}$. If the \gp
moments are computed exactly, replace $\hat Q_{\mathbf X}$ by
$Q_{\mathbf X}$. Since the locations are redrawn in every step, the
optimized quantity is
$\mathbb E_{\mathbf X \sim \mathcal G^N}\bigl[\frac{\mathrm{MMD}_2^2(\hat P_{\mathbf X}, \hat Q_{\mathbf X})}{2N}\bigr]$,
which is the empirical counterpart of $\frac{\mathcal L_2}{2N}$.
\end{remark}

\subsection{Preserving expressiveness}
\label{app:theoReExpr}
For simplicity, we repeat the \ucda theorem by \citep{sharma2023partiallystoch} again:

\begin{restatable}[Universal Conditional Distribution with Finite Stochasticity]{theorem}{Sharma}
\label{theoremA}
Let $\boldsymbol X$ take values in a compact $\mathcal{X} \subset \mathbb{R}^{d_0}$ and $\boldsymbol Y$ in $\mathcal{Y} \subset \mathbb{R}^{d_L}$. Let $\boldsymbol f_\upsilon (\boldsymbol Z,\boldsymbol x) $ be a network with deterministic parameters $\upsilon$ and Gaussian noise $\boldsymbol Z \in \mathbb{R}^m$ independent of $\boldsymbol X$ with finite mean and variance. Suppose that for some hidden layer $h_\ell$, after restricting $\boldsymbol Z$ to a ball $ \|\boldsymbol Z \| \leq R$, \textup{(1)} $\boldsymbol Z$ and $\boldsymbol x$ are fully input by this layer, \textup{(2)} $\boldsymbol h_\ell$ ranges over a compact set, \textup{(3)} there is a deterministic, continuous, surjective $\boldsymbol g$ with $\boldsymbol g(\boldsymbol h_\ell) = (\boldsymbol Z,\boldsymbol x)$, and \textup{(4)} the downstream network is a universal approximator in the sense of \citep{Leshno1993UAT}. If $\boldsymbol Y|\boldsymbol X$ admits a continuous generator, then $\boldsymbol f_\upsilon$ approximates $\boldsymbol Y|\boldsymbol X$ arbitrarily well.
\end{restatable}

\DetCert*

\begin{proof}
Let $J_0 := \{j:b_j^{(1)} \in S\}$ and fix $J \subseteq J_0$ with $|J| = d_L$. Units in $T$ carry no stochastic parameter, so $J \cap T = \emptyset$, and $d_1 \geq d_0+d_L$ follows. Let $E:= \{1, \dots,d_1\}\setminus(T\cup J)$.

\textit{Step 1: reduced network} Consider configurations in which every weight from a unit in $E$ into layer 2 has mean $0$, and every stochastic parameter in layers $\geq 2$ has variance $v \geq 0$. Denote the output by $\boldsymbol F_v(\mathbf x)$. For $v = 0$ the output depends on the first layer through the units in $T\cup J$. For $j \in J$ decompose the pre-activation as $u_j = \boldsymbol {\bar{W}}_{j,\cdot}\boldsymbol {x} + \bar{b}_j + {\xi}_j$, where 
\[
    \xi_j = \sum_{i: (j,i)\in S}\sigma_{j,i} \varepsilon_{j,i} x_i + \sigma_j^b \varepsilon_j^b, \quad s_j(x)^2:= \sum_{i: (j,i)\in S} \sigma_{j,i}^2 x_i^2 + (\sigma_j^b)^2 \geq  ( \sigma_j^b)^2 >0
\]
with independent standard Gaussian $\boldsymbol {\varepsilon}$ that are independent of $\mathbf{X}$. Set $Z_j:= \frac{\xi_j}{s_j(\mathbf{X})}$ for $j \in J$. Conditionally on $\mathbf{X}=\mathbf{x}$ the vector $\mathbf{Z}_J$ is $\mathcal{N}(\mathbf{0}, \mathbf{I}_{d_L})$, and since this distribution does not depend on $\mathbf{x}$, $\mathbf{Z}_J$ is independent of $\mathbf{X}$. The first layer output on $T\cup J$ is $\mathbf{h} = \varphi(\mathbf{u})$ with $\mathbf{u}_T = \mathbf{\bar{W}}_{T, \cdot}\mathbf{x} + \mathbf{\bar{b}}_T$ and $u_j = \mathbf{\bar{W}}_{j, \cdot} \mathbf{x} + \bar{b}_j + s_j(\mathbf{x})Z_j$, a continuous function of $(\mathbf{Z}_J,\mathbf{x})$. We verify the conditions of \Cref{theoremA} for $\mathbf{F}_0$ with $\ell = 1$ and $m = d_L$. \textup{(1)} Holds by construction, \textup{(2)} $\mathbf{h}$ is a continuous image of the compact set $\{\| \mathbf{Z}_J\| \leq R \} \times \mathcal{X}$. \textup{(3)} A continuous strictly monotone $\varphi$ has a continuous inverse on the interval $\varphi(\mathbb{R})$, so $\mathbf u = \varphi^{-1}(\mathbf h)$ is recovered continuously. Since $|T| \geq d_0$ and the deterministic parameters are free, we may choose their values such that $\mathrm{rank}(\mathbf{\bar{W}}_{T,\cdot})=d_0$. Now we recover $\mathbf{x} = \mathbf{\bar{W}}_{T, \cdot}^{+}(\mathbf{u}_T - \mathbf{\bar{b}}_T)$, then $Z_j =\frac{(u_j - \mathbf{\bar{W}}_{j, \cdot}\mathbf{x}-\bar{b}_j)}{s_j(\mathbf{x})}$ for $j \in J$, where $s_j$ is continuous and bounded from below. This map $\mathbf{g}$ is continuous, recovers $(\mathbf{Z}_J,\mathbf{x})$ exactly and is surjective onto the restricted domain. \textup{(4)} Since $L \geq 3$, the network from layer 1 onward has at least one hidden layer with a continuous non-polynomial activation, and its weights out of $T\cup J$ are free. \citep{Leshno1993UAT} applies. Hence for every $\varepsilon>0$, there are parameter values with $\sup_\mathbf{x}d_{\mathrm{BL}}(p\mathbf{F}_0(\mathbf{x}), p(\mathbf{Y}|\mathbf{X}=\mathbf{x})) \leq \frac{\varepsilon}{2}$.

\textit{Step 2: Downstream noise} Write $\boldsymbol F_v(\boldsymbol x) = \boldsymbol \Psi(\boldsymbol x, \boldsymbol Z, \sqrt{v}\boldsymbol \varepsilon')$, where $\boldsymbol Z$ collects all first-layer noise, $\boldsymbol \varepsilon'$ the standard Gaussian noise of the stochastic parameters in layers $\geq 2$, and $\boldsymbol \Psi$ is jointly continuous. Then $\boldsymbol \Psi(\boldsymbol x,\boldsymbol Z,\boldsymbol 0) = \boldsymbol F_0(\boldsymbol x)$ and 
\[
    d_{\mathrm{BL}}(p\boldsymbol F_v(\boldsymbol x), p\boldsymbol F_0(\boldsymbol x)) \leq \mathbb{E}[\min\{2, \| \boldsymbol \Psi(\boldsymbol x, \boldsymbol Z, \sqrt{v}\boldsymbol \varepsilon') - \boldsymbol \Psi(\boldsymbol x,\boldsymbol Z,\boldsymbol 0)\|\}].
\]
Given $\delta>0$, choose $R'$ with $\mathbb{P}(\|(\mathbf{Z}, \boldsymbol{\varepsilon'})\|>R') \leq \delta$. On the compact set $\mathcal{X} \times \{ \|(\mathbf{z},\mathbf{e})\| \leq R'\}$, the map $\boldsymbol{\Psi}$ is uniformly continuous in $\mathbf{x}$. Choosing $3 \delta \leq \frac{\varepsilon}{2}$ and $v >0$ accordingly, gives a nondegenerate Gaussian distribution on $\Theta_s$ with $\sup_\mathbf{x} d_{\mathrm{BL}} (p\mathbf{F}_v(\mathbf{x}), p(\mathbf{Y}|\mathbf{X}=\mathbf{x})) \leq \varepsilon$. 

\textit{Complexity} One pass over the bias mask gives $U(S)$, one pass over the weight mask gives $T$. Together this adds up to $|\Theta^{(1)}|$
\end{proof}

\begin{restatable}[Make $\boldsymbol {\bar{W}}_{T, \cdot}$ full rank]{proposition}{WeightRepair}
\label{Prop:WeightRepair}
For fixed values of the deterministic first-layer weights, the rank condition can be restored as follows. For $j \in \{1,\dots,d_1 \}$ let $w_j:=|S \cap \Theta^{(1)}_j|$, where $\Theta_j^{(1)}$ denotes the first-layer parameters of unit $j$, be the number of parameters of unit $j$ that survived the prior optimization. Making unit $j$ fully deterministic means setting $\sigma_\theta=0$ for all $\theta \in S \cap\Theta^{(1)}_j$ changing $w_j$ parameters.

Let $J\subseteq\{1, \dots, d_1 \}$ be any set of units excluded from the repair and put $A :=\{1, \dots, d_1\}\setminus J$. Suppose $\mathrm{rank}(\boldsymbol {\bar{W}}^{(1)}_{A,\cdot}) = d_0$. Then \textup{(i)} the family $\mathcal{I}:=\{B\subseteq A: \mathrm{rank}(\boldsymbol {\bar{W}}^{(1)}_{B, \cdot}) = |B|\}$ is the set of independent sets of a linear matroid on $A$, and every basis has exactly $d_0$ elements; \textup{(ii)} sorting $A$ by $w_j$ in increasing order and greedily adding a unit whenever it increases the rank returns a basis $B^\star$ of minimum total weight $\sum_{j \in B^\star}w_j$; \textup{(iii)} making every unit of $B^\star$ fully deterministic yields a mask whose set $T'$ of fully deterministic units contains $B^\star$, so that $\mathrm{rank}(\boldsymbol {\bar{W}}^{(1)}_{T', \cdot}) = d_0$; \textup{(iv)} no repair that leaves the units of $J$ untouched and restores the rank condition changes fewer parameters if no unit in $J$ is fully deterministic. The cost is $\mathcal{O}(d_1\log d_1 + d_1d_0^2)$
\end{restatable}

\begin{proof}
    \textup{(i)} Linear independence of rows of a fixed matrix defines a linear (matric) matroid on the row index set with $\mathrm{rank}(\boldsymbol {\bar{W}}_{A, \cdot}^{(1)}) = d_0$, so all bases have $d_0$ elements. \textup{(ii)} The weight $B \mapsto \sum_{j \in B}w_j$ is nonnegative and modular, and for such weights the greedy algorithm returns a minimum-weight basis of a matroid \citep{Welsh1970Matroid}. \textup{(iii)} Every parameter of a unit in $B^\star$ has $\sigma_\theta = 0$ after the repair, so $B^\star \subseteq T'$, and the rank of a submatrix does not decrease when rows are added. \textup{(iv)} Any repaired mask satisfying the certificate has some $T'$ with $\mathrm{rank}(\boldsymbol{\bar{W}}^{(1)}_{T', \cdot}) = d_0$, hence $T'\cap A$ contains a basis $B'$ of the matroid, and every surviving parameter of every unit of $B'$ must have been set to zero, which costs $\sum_{j \in B'}w_j \geq \sum_{j \in B^\star}w_j$. The complexity of it is one sort plus $d_1$ rank-one updates of QR factorization of size $d_0$.
\end{proof}

\begin{restatable}[Minimal Repair]{corollary}{MinRep}
\label{cor:MinRep}
Suppose $d_1 \geq d_0 + d_L$. Let $J_0 := \{j : b_j^{(1)} \in S\}$, so that $U(S) = |J_0|$, and for every first-layer unit $j$ let $w_j := |S \cap \Theta^{(1)}_j|$ be the number of its stochastic parameters.

\textup{Step 1 (coverage).} If $U(S) < d_L$, choose $J \supseteq J_0$ with $|J| = d_L$ and make the biases of $J \setminus J_0$ stochastic. Otherwise choose $J \subseteq J_0$ with $|J| = d_L$. This changes $\max\{0, d_L - U(S)\}$ parameters and gives $U(S) \geq d_L$, and no repair achieving $U(S) \geq d_L$ changes fewer.

\textup{Step 2 (deterministic units).} Let $T$ be the set of fully deterministic units after Step 1 and $A := \{1, \dots, d_1\} \setminus J$. If $|T| < d_0$, let $B^\star$ consist of the $d_0 - |T|$ units of $A \setminus T$ with the smallest $w_j$, and make all first-layer parameters of these units deterministic. This changes $\sum_{j \in B^\star} w_j$ parameters, and no repair that leaves $J$ untouched and achieves $|T'| \geq d_0$ changes fewer, where $T' := T \cup B^\star$.

The resulting mask satisfies the certificate of Proposition~\ref{prop:certificate}. Step 2 does not undo Step 1 because $B^\star \cap J = \emptyset$, and Step 1 does not undo Step 2 because the units of $J$ have stochastic biases and are therefore not in $T'$.
\end{restatable}

\begin{proof}
Step 1 makes one bias stochastic per changed parameter, which gives $U(S) \geq d_L$ afterwards, and any repair achieving this must change at least $d_L - U(S)$ biases. For Step 2, the units of $J$ have stochastic biases, so $T \subseteq A$ and $|A \setminus T| = |A| - |T| \geq d_0 - |T|$ by $d_1 \geq d_0 + d_L$. Hence $B^\star$ exists. Any repair that leaves $J$ untouched and achieves $|T'| \geq d_0$ must make at least $d_0 - |T|$ units of $A \setminus T$ fully deterministic, and doing so for a unit $j$ changes exactly $w_j$ parameters. The sum of $w_j$ over $d_0 - |T|$ units is smallest for the units with the smallest $w_j$, which are those in $B^\star$. After both steps $U(S) \geq d_L$ and $|T'| \geq d_0$, so Proposition~\ref{prop:certificate} applies.
\end{proof}

\subsection{Inference}
\label{app:theoReInf}
\begin{assumption}[Regularity]
\label{ass:reg}
\textup{(i)} For every $\boldsymbol \theta_s$, the map $\boldsymbol \theta_d \mapsto \ell (\boldsymbol \theta_d, \boldsymbol \theta_s)$ is continuously differentiable.
\textup{(ii)} There are $r \geq 0$ and $G\in L^1(p(\boldsymbol \theta_s))$,  with
$\|\nabla_{\boldsymbol \theta_d}\ell(\boldsymbol \theta_d,\boldsymbol \theta_s)\|\leq G(\boldsymbol \theta_s) (1+ \| \boldsymbol  \theta_d\|)^r$ for all $\boldsymbol \theta_d, \boldsymbol \theta_s$.
\textup{(iii)} $B := \sup _{\boldsymbol \theta_d,\boldsymbol \theta_s} e^{\ell(\boldsymbol \theta_d,\boldsymbol \theta_s)} < \infty$ and $Z(\boldsymbol \theta_d) >0$ for all $\boldsymbol \theta_d$.
\end{assumption}
\begin{remark}
\label{remark:regAss}
    \textup{(iii)} holds for categorical likelihoods and for Gaussian likelihoods with noise variance bounded from below. \textup{(ii)} holds whenever the gradient grows polynomially in $\boldsymbol \theta_s$, since $p(\boldsymbol \theta_s)$ is Gaussian. Under \textup{(ii)} and \textup{(iii)}, dominated convergence shows that $Z$ is $C^1$ and that $\boldsymbol \theta_d \mapsto \int \varphi(\boldsymbol \theta_d,\boldsymbol \theta_s) e^{\ell(\boldsymbol \theta_d,\boldsymbol \theta_s)}p(\boldsymbol \theta_s)d\boldsymbol \theta_s$ is continuous for every bounded continuous $\varphi$.
\end{remark}
Let $\Pi_\tau$ denote the joint posterior of $(\boldsymbol \theta_d, \boldsymbol \theta_s)$ under the prior $\mathcal N(\boldsymbol {\hat\mu}_d, \tau^2 \boldsymbol I_k) \otimes p(\boldsymbol \theta_s)$, so that $\Pi_\tau(\mathrm d\boldsymbol \theta_d, \mathrm d\boldsymbol \theta_s) \propto e^{\ell(\boldsymbol \theta_d, \boldsymbol \theta_s)}\, \mathcal N(\boldsymbol \theta_d; \boldsymbol {\hat\mu}_d, \tau^2\boldsymbol  I_k)\, p(\boldsymbol \theta_s)\, \mathrm d\boldsymbol \theta_d\, \mathrm d\boldsymbol \theta_s$.
\begin{restatable}[$\tau$ interpolates between freezing and type-II ML]{proposition}{TauProp}
\label{prop:gamma}
Under Assumption~\ref{ass:reg} the joint posterior $ \Pi_\tau$ has a Lebesgue density for every $\tau >0$. Moreover, \textup{(i)} as $\tau \downarrow 0, \Pi_\tau$ converges weakly to $\delta_{\boldsymbol {\hat{\mu}}_d} \otimes p(\boldsymbol \theta_s|\mathcal{D}, \boldsymbol {\hat{\mu}}_d)$; \textup{(ii)} for $0  < \tau < \infty, J_\tau$ is, up to a constant, the log marginal posterior density of $\boldsymbol \theta_d$; \textup{(iii)} as $\tau \uparrow\infty, J_\tau \rightarrow\log Z$ uniformly on compacts. Keeping the degenerate prior $\delta_{\boldsymbol {\hat{\mu}}_d}$ while optimizing $\boldsymbol \theta_d$ is ill-posed, since the log-prior is $-\infty$ off the atom.
\end{restatable}

\begin{proof}
The joint density $\frac{e^{\ell(\boldsymbol \theta_s, \boldsymbol \theta_d)}p(\boldsymbol \theta_s)\mathcal{N}(\boldsymbol \theta_d; \boldsymbol {\hat{\mu}}_d, \tau^2\boldsymbol I)}{p(\mathcal{D})}$ is integrable by regularity assumption \textup{(iii)}. For \textup{(i)} let $\varphi$ be bounded and continuous and set $\Phi(\boldsymbol \theta_d) := \int \varphi(\boldsymbol \theta_d,\boldsymbol \theta_s)e^{\ell(\boldsymbol \theta_d,\boldsymbol \theta_s)}p(\boldsymbol \theta_s)d\boldsymbol \theta_s$. $\Phi$ and $Z$ are bounded by $\|\varphi\|_\infty B$ and $B$ and continuous by the Remark~\ref{remark:regAss}. Since $\mathcal{N}(\boldsymbol {\hat{\mu}}_d, \tau^2\boldsymbol I) \rightarrow \delta_{\boldsymbol {\hat{\mu}}_d}$,
\[
    \mathbb{E}_{\Pi_\tau \varphi} = \frac{\int \Phi d \mathcal{N}(\boldsymbol {\hat{\mu}}_d, \tau^2\boldsymbol I)}{\int Z d \mathcal{N}(\boldsymbol {\hat{\mu}}_d, \tau^2\boldsymbol I)} \rightarrow \frac{\Phi(\boldsymbol {\hat{\mu}}_d)}{Z(\boldsymbol {\hat{\mu}}_d)} = \mathbb{E}_{p(\boldsymbol \theta_s|\mathcal{D}, \boldsymbol {\hat{\mu}}_d)}[\varphi(\boldsymbol {\hat{\mu}}_d,\boldsymbol \theta_s)],
\]
using $Z(\boldsymbol {\hat{\mu}}_d) >0$. \textup{(ii)} follows by writing out $\log Z(\boldsymbol \theta_d) + \log \mathcal{N}(\boldsymbol \theta_d; \boldsymbol {\hat{\mu}}_d, \tau^2\boldsymbol I)$. For \textup{(iii)}, $\sup_{\|\boldsymbol \theta_d\|<\rho}\frac{\|\boldsymbol \theta_d - \boldsymbol {\hat{\mu}}_d \|^2}{2\tau^2} \rightarrow 0$.
\end{proof}

\InfScheme*

\begin{proof}
By definition of the \kl divergence,
\[
\mathrm{KL} (q_\varepsilon \otimes q_s\| \Pi_\tau) = - H(q_\varepsilon)- H(q_s) - \mathbb{E}_{ q_s \otimes q_\varepsilon}[\ell (\boldsymbol \theta'_d, \boldsymbol \theta_s) + \log p(\boldsymbol \theta'_d) + \log p(\boldsymbol \theta_s)] + \log  p(\mathcal{D}),
\]
and $H(q_\varepsilon) = \frac{k}{2}\log (2\pi e \varepsilon^2)$. Exactly $\mathbb{E}_{ q_\varepsilon}[\log p(\boldsymbol \theta'_d)] = \log p(\boldsymbol \theta_d) - \frac{k\varepsilon^2}{2\tau^2}$. By the mean value theorem and \textup{(ii)} of Assumption~\ref{ass:reg}, $|\ell (\boldsymbol \theta'_d,\boldsymbol \theta_s) - \ell (\boldsymbol \theta_d,\boldsymbol \theta_s)| \leq G(\boldsymbol \theta s) (1+ \| \boldsymbol \theta_d \| + \| \boldsymbol \theta'_d - \boldsymbol \theta_d \|)^r \| \boldsymbol \theta'_d - \boldsymbol \theta_d\|$, so
\[
|\mathbb{E}_{q_\varepsilon \otimes q_s}[\ell (\boldsymbol \theta'_d,\boldsymbol \theta_s)] - \mathbb{E}_{q_s}[\ell (\boldsymbol \theta_d,\boldsymbol \theta_s)]| \leq \mathbb{E}_{q_s}[G]\mathbb{E}_{\boldsymbol W \sim \mathcal{N}(\boldsymbol 0, \boldsymbol I_k)}[(1+\|\boldsymbol \theta_d \| + \| \boldsymbol W \|)^r \| \boldsymbol W \|]\varepsilon,
\]
which defines $C(\boldsymbol \theta_d)$. Using $p(\boldsymbol \theta_s|\mathcal{D}, \boldsymbol \theta_d) = \frac{e^{\ell(\boldsymbol \theta_d,\boldsymbol \theta_s)}p(\boldsymbol \theta_s)}{Z(\boldsymbol \theta_d)}$,
\[
\mathrm{KL}(q_s \| p(\cdot,| \mathcal{D}, \boldsymbol \theta_d)) = - H (q_s) - \mathbb{E}_{q_s}[\ell(\boldsymbol \theta_d,\boldsymbol \theta_s)] - \mathbb{E}_{q_s}[\log p(\boldsymbol \theta_s)] + \log Z(\boldsymbol \theta_d),
\]
and substituting gives the stated form of $F$. The \kl term is nonnegative and vanishes if and only if $q_s = p (\cdot |\mathcal{D}, \boldsymbol \theta_d)$, which lies in $\mathcal{P}_G$ by \textup{(iii)}. It is the only term that depends on $q_s$, so minimizing $F$ reduces to maximizing $J_\tau$.
\end{proof}

\begin{restatable}[Tracking bound under constant step sizes]{theorem}{TrackBound}
\label{thm:track}
Assume for all $\boldsymbol \theta_d, \boldsymbol \theta_d' \in K$ \textup{(C1)} $W_1 ( \mu P^h_{\boldsymbol \theta_d}, \nu P_{\boldsymbol \theta_d}^h) \leq \rho W_1(\mu, \nu)$ with $\rho <1$, \textup{(C2)} $W_1(\pi_{\boldsymbol \theta_d}^h, \pi_{\boldsymbol \theta_d'}^h) \leq L_\pi \| \boldsymbol \theta_d - \boldsymbol \theta'_d \|$, \textup{(C3)} $\|\boldsymbol \theta_{d,t+1}-\boldsymbol \theta_{d,t}\| \leq \delta_t$, \textup{(C4)} $W_1(\pi_{\boldsymbol \theta_d}^h, \pi_{\boldsymbol \theta_d}) \leq C_1h$ for the $s$-marginals. Let $\boldsymbol X_t = (\boldsymbol \theta_{s,t}, v_t)$ denote the state of the sampler after $t$ steps, with momentum $v_t$. Let $\mu_t$ be the distribution of $\boldsymbol X_t$ and $e_t:= W_1(\mu_t, \pi_{\boldsymbol \theta_{d,t}}^h)$. Then 
\begin{align*}
e_{t+1} \leq \rho e_t + L_\pi \delta_t&, \quad e_t \leq \rho^te_0 + L_\pi \sum_{j<t}\rho^{t-1-j}\delta_j, \\ W_1(\boldsymbol \mu_t^{\boldsymbol \theta_s}, \pi_{\boldsymbol \theta_{d_t}}) &\leq e_t + C_1h
\end{align*}
\end{restatable}

\begin{proof}
Since $\pi_{\boldsymbol \theta_{d,t}}^h$ is invariant for $P_{\boldsymbol \theta_{d,t}}^h$ and $\mu_{t+1} = \mu_t P_{\boldsymbol \theta_{d,t}}^h$ for an open-loop path,
\[
e_{t+1} \leq W_1(\mu_t P_{\boldsymbol \theta_{d,t}}^h, \pi^h_{\boldsymbol \theta_{d,t}}P_{\boldsymbol \theta_{d,t}}^h)+ W_1(\pi_{\boldsymbol \theta_{d,t}}^h, \pi_{\boldsymbol \theta_{t+1}}^h) \leq \rho e_t + L_\pi \delta_t
\]
Iterating gives the second bound. The third follows from the triangle inequality, marginalization, and \textup{(C4)}.
\end{proof}

\CoupledStepSize*

\begin{proof}
Insert $\delta_j \leq G\eta$ into \Cref{thm:track} and use $\sum_j\rho^j\leq \frac{1}{1-\rho}$. For (b) add $C_1h$ from \textup{(C4)}.
\end{proof}

This result gives an upper bound for the error, the following proposition shows that this bound is also attained:
\begin{restatable}[The coupled bound is attained]{proposition}{BoundAttained}
\label{prop:BoundAttained}
Let $\pi_{\theta_d} = \mathcal{N}(\theta_d,1)$ on $\mathbb{R}$, let $P_{\theta_d}^h$ be the unadjusted Langevin kernel $x \mapsto x - h(x-\theta_d) + \sqrt{2h}W$ with $h \in (0,1)$, let $X_0 \sim \pi_{\theta_{d,0}}^h$ and let $\theta_{d,t} = \theta_{d,0} + G\eta t$. Then \textup{(C1)} to \textup{(C3)} hold with $\rho = 1-h, c = 1, L_\pi = 1, \delta_t = G\eta$, and 
\[
W_1(\mu_t, \pi_{\theta_{d,t}}) \geq W_1(\mu_t, \pi_{\theta_{d,t}}^h) = \frac{G\eta}{h} ( 1- (1-h)^t).
\]
In the coupled case $h = \eta$ this is at least $G(1-e^{-1})$ for all $t \geq \frac{1}{\eta}$, independently of $\eta$.
\end{restatable}

\begin{proof}
Synchronous coupling gives \textup{(C1)} with $\rho = 1-h$. The variance recursion $v' = (1-h)^2v + 2h$ has fixed point $v_h = \frac{1}{1- \frac{h}{2}}$, so $\pi_{\theta_d}^h =\mathcal{N}(\theta_d,v_h)$ and \textup{(C2)} holds with $L_\pi = 1$. Since $X_0 \sim \pi_{\theta_{d,0}}^h$ and the variance recursion does not involve $\theta_d, \mathcal{N}(m_t, v_t)$ with $m_{t+1} = (1-h)m_t + h\theta_{d,t}$. The lag $\ell_t := \theta_{d,t} -m_t$ satisfies $\ell_0 = 0$ and $\ell_{t+1} = (1-h)\ell_t + G \eta$, hence $\ell_t = \frac{G\eta}{h}(1- (1-h)^t)$. Two Gaussians with equal variance and means differing by $\ell_t$ are at $W_1$-distance $\ell_t$, and $W_1(\mu_t, \pi_{\theta_{d,t}}) \geq |\mathbb{E}[X_t] - \theta_{d,t}|= \ell_t$. Finally, $(1-\eta)^{\frac{1}{\eta}} \leq e^{-1}$.
\end{proof}

We also manage to find an error decomposition, where each term relates to a specific modeling choice. This allows us to better understand how to make the modeling choices, if we want to reduce the error.
\begin{restatable}[Error decomposition]{corollary}{ErrorDecomp}
\label{cor:decomp}
In the setting of \Cref{thm:track}, let $\phi$ satisfy $\|\phi\|_\infty \leq 1$ and $\mathrm{Lip}(\phi) \leq 1$, let $T_0$ be the burn-in, $\hat{A}:= \frac{1}{T} \sum_{t = T_0+1}^{T_0+T} \phi(\mathbf X_t^s), \bar{\delta}:= \sup_t \delta_t$, and $M:= \sup_t \mathbb{E}[\|\mathbf X_t - \mathbf X_t' \|] < \infty$ with $\mathbf X_t'$ an independent copy of $\mathbf X_t$. Then
\begin{align*}
\left( \mathbb{E}[(\hat{A}- \mathbb{E}_{\pi_{\boldsymbol \theta_d^\star}}[\phi])^2]\right)^{\frac{1}{2}} \leq \underbrace{C_1h}_{\text{(a) Discretization}} + \underbrace{\frac{L_\pi \bar{\delta}}{1-\rho}}_{\text{(b) Tracking}} \\+ \underbrace{\frac{L_\pi}{T} \sum_{t = T_0+1}^{T_0+T}\| \boldsymbol \theta_{d_t} - \boldsymbol \theta^\star_d \|}_{\text{(c) Optimizer}} + \underbrace{\frac{\rho^{T_0}e_0}{T(1-\rho)}}_{\text{(d) Burn-in}} + \underbrace{\sqrt{\frac{2 \max\{1,M\}}{T(1-\rho)}}}_{\text{(e) Monte Carlo}}
\end{align*}
\end{restatable}

\begin{proof}
\textit{Bias.} Decompose 
\[
    \mathbb{E}[\hat{A}]- \pi_{\boldsymbol \theta^\star}\phi = \frac{1}{T} \sum_t (\mu_t\phi - \pi_{\boldsymbol \theta_{d,t}}^h\phi) + \frac{1}{T} \sum_t (\pi_{\boldsymbol \theta_{d,t}}^h \phi - \pi_{\boldsymbol \theta^\star}^h\phi) + (\pi_{\boldsymbol \theta^\star}^h\phi - \pi_{\boldsymbol \theta^\star}\phi)
\]
Since $\phi$ is $1$-Lipschitz, the first sum is at most $\frac{1}{T} \sum_t e_t \leq \frac{\rho^{T_0}e_0}{T(1-\rho)}+ \frac{L_\pi \bar{\delta}}{1-\rho}$ by \Cref{thm:track}, the second is at most term (c) by \textup{(C2)}, and the third is at most $C_1h$ by \textup{(C4)}.

\textit{Variance.} For $u <t$ let $g := \phi\Pi _{i = u}^{t-1}P_{\theta_i}^h$. By \textup{(C1)} applied to Dirac measures, $\mathrm{Lip} (g) \leq \rho^{t-u}$. Hence, 
\[
|\mathrm{Cov}(\phi(X_u), \phi(X_t))| = |\mathrm{Cov}(\phi(X_u), g(X_u))| \leq \mathbb{E}[|g(X_u) - \mathbb{E}[g(X_u)]|] \leq \rho^{t-u}M,
\]
and trivially $|\mathrm{Cov}|\leq1$. Therefore, $\mathrm{Var}(\hat{A}) \leq \frac{\max \{ 1,M\}}{T^2} \sum_{u,t} \rho^{|t-u|} \leq \frac{2 \max \{ 1,M\}}{T(1-\rho)}$. Minkowski's inequality combines bias and standard deviation.
\end{proof}

\begin{corollary}[Drift radius]
\label{cor:DriftRadius}
Every stationary point of $J_\tau$ satisfies $\boldsymbol \theta^\star_d-\boldsymbol {\hat{\mu}}_d = \tau^2 \nabla \log Z(\boldsymbol \theta^\star_d)$. If, in addition $\nabla \log Z$ is $L_g$-Lipschitz and $\tau^2L_g<1$, then $J_\tau$ is strongly concave, $\boldsymbol \theta^\star_d$ is unique, and with $g_0:= \nabla \log Z(\boldsymbol {\hat{\mu}}_d)$,
\[
\frac{\tau^2 \|\boldsymbol g_0 \|}{1 + \tau^2L_g} \leq \| \boldsymbol \theta^\star_d - \boldsymbol {\hat{\mu}}_d \| \leq \frac{\tau^2 \|\boldsymbol g_0 \|}{1 - \tau^2L_g},
\]
so the drift equals $\tau^2 \|\boldsymbol g_0\| (1+ \mathcal{O}(\tau^2))$. By Fisher's identity, $\boldsymbol g_0$ is estimable from the frozen run of regime $(i)$ at no extra cost.
\end{corollary}

\begin{proof}
The identity is the stationary condition $\nabla \log Z(\boldsymbol \theta^\star_d) - \tau^{-2} (\boldsymbol \theta^\star_d - \boldsymbol {\hat{\mu}}_d )= 0$. An $L_g$-Lipschitz gradient makes $\frac{L_g}{2}\|\boldsymbol \theta\|^2-\log Z$ convex, so $J_\tau$ is strongly concave with modulus $\tau^{-2}-L_g >0$. From the identity, $\| \boldsymbol \theta^\star_d - \boldsymbol {\hat{\mu}}_d\| = \tau^2 \| \nabla \log Z(\boldsymbol \theta^\star_d)\|$ and $| \|\nabla \log Z(\boldsymbol \theta^\star_d) \| - \|\boldsymbol g_0\| | \leq L_g \| \boldsymbol \theta^\star_d - \boldsymbol {\hat{\mu}}_d \|$; rearrange.
\end{proof}
\section{Detailed \texttt{UCI} results}
First we are going to provide some information on the datasets. The licenses of the datasets can be found on their respective websites.
\begin{table}[H]
\centering
\small
\begin{tabular}{lcrrl}
\toprule
\textbf{Dataset} & $(d_0, d_L)$ & \textbf{Sample size} & \textbf{Parameters} & \textbf{Citation} \\
\midrule
\texttt{CONCRETE}       & (8, 1)  & 1030  & 11101 & \citep{Yeh1998CONCRETE} \\
\texttt{BOSTON}         & (13, 1) & 506   & 11601 & \citep{Harrison1978Boston} \\
\texttt{BIKESHARING}    & (12, 1) & 731   & 11501 & \citep{Hadi2013Bikesharing} \\
\texttt{WINE}           & (10, 1) & 4898  & 11301 & \citep{Cortez2009Wine} \\
\texttt{ONLINESHOPPERS} & (17, 2) & 12330 & 12102 & \citep{Sakar2018ONLINESHOPPERS} \\
\texttt{CREDIT}         & (23, 2) & 30000 & 12702 & \citep{Yeh2009Credit} \\
\texttt{BANKNOTE}       & (4, 2)  & 1372  & 10802  & \citep{Lohweg2012Banknote} \\
\texttt{HTRU2}          & (8, 2)  & 17898 & 11202 & \citep{Lyon2015HTRU2} \\
\bottomrule
\end{tabular}
\caption{Dataset characteristics.}
\label{tab:datasets}
\end{table}
\label{App:AddRes}
\begin{table}[H]
\centering
\caption{Results on \texttt{BANKNOTE} ($\mathrm{mean}\pm\mathrm{std}$).}
\label{tab:cla_banknote}
\small
\setlength{\tabcolsep}{3pt}
\begin{tabular}{ccccc}
\toprule
& \acc in \% $(\uparrow)$ & \nll $(\downarrow)$ & \ece $(\downarrow)$& Deterministic parameters in \%\\
\midrule
DWF-PSNN& \textbf{100.00 $\pm$ 0.00} & 0.009 $\pm$ 0.002& 0.009 $\pm$ 0.002 &  55.32 $\pm$ 2.56\\
ALS-BNN& \textbf{100.00 $\pm$ 0.00} & \textbf{0.004 $\pm$ 0.001} & \textbf{0.004 $\pm$ 0.000}& 0.00\\
FLS-BNN &  \textbf{100.00 $\pm$ 0.00} & \textbf{0.004 $\pm$ 0.001} & \textbf{0.004 $\pm$ 0.001}& 95.37\\
RAND-PSNN& \textbf{100.00 $\pm$ 0.00}& 0.005 $\pm$ 0.002 & 0.005 $\pm$ 0.001 & 55.32 $\pm$ 2.56\\
SNI-PSNN& 99.20 $\pm$ 0.63& 0.032 $\pm$ 0.010 & 0.025 $\pm$ 0.007 & 90.00\\
\bottomrule
\end{tabular}
\end{table}

\begin{table}[H]
\centering
\caption{Results on \texttt{HTRU2} ($\mathrm{mean}\pm\mathrm{std}$).}
\label{tab:cla_htru2}
\small
\setlength{\tabcolsep}{3pt}
\begin{tabular}{ccccc}
\toprule
& \acc in \% $(\uparrow)$ & \nll $(\downarrow)$ & \ece $(\downarrow)$& Deterministic parameters in \%\\
\midrule
DWF-PSNN& \textbf{97.83 $\pm$ 0.48} & \textbf{0.069 $\pm$ 0.016}& \textbf{0.007 $\pm$ 0.003} &  57.31 $\pm$ 2.79\\
ALS-BNN& 97.59 $\pm$ 0.50 & 0.078 $\pm$ 0.016 & 0.008 $\pm$ 0.003& 0.00\\
FLS-BNN& 97.56 $\pm$ 0.56 & 0.079 $\pm$ 0.015 & 0.008 $\pm$ 0.003& 91.97 \\
RAND-PSNN& 97.82 $\pm$ 0.49 & 0.070 $\pm$ 0.015 & \textbf{0.007 $\pm$ 0.003} & 57.31 $\pm$ 2.79\\
SNI-PSNN& 97.41 $\pm$ 0.58 & 0.086 $\pm$ 0.016 & 0.009 $\pm$ 0.002& 90.00\\
\bottomrule
\end{tabular}
\end{table}

\begin{table}[H]
\centering
\caption{Results on \texttt{ONLINESHOPPERS} ($\mathrm{mean}\pm\mathrm{std}$).}
\label{tab:cla_onlineshoppers}
\small
\setlength{\tabcolsep}{3pt}
\begin{tabular}{ccccc}
\toprule
& \acc in \% $(\uparrow)$ & \nll $(\downarrow)$ & \ece $(\downarrow)$& Deterministic parameters in \%\\
\midrule
DWF-PSNN& 89.43 $\pm$ 0.64 & \textbf{0.243 $\pm$ 0.020}& 0.016 $\pm$ 0.005 &  61.01 $\pm$ 3.43\\
ALS-BNN& 89.25$\pm$ 0.74 & 0.261 $\pm$ 0.019 & 0.024 $\pm$ 0.009& 0.00\\
FLS-BNN& 89.25 $\pm$ 0.85 & 0.256 $\pm$ 0.020 & 0.018 $\pm$ 0.006 & 85.13 \\
RAND-PSNN& \textbf{89.48 $\pm$ 0.69} & \textbf{0.243 $\pm$ 0.020} & \textbf{0.015 $\pm$ 0.005}& 60.91 $\pm$ 3.44\\
SNI-PSNN& 88.70 $\pm$ 0.90 & 0.292 $\pm$ 0.017 & 0.024 $\pm$ 0.012 & 90.00\\
\bottomrule
\end{tabular}
\end{table}

\begin{table}[H]
\centering
\caption{Results on \texttt{CREDIT} ($\mathrm{mean}\pm\mathrm{std}$).}
\label{tab:cla_credit}
\small
\setlength{\tabcolsep}{3pt}
\begin{tabular}{ccccc}
\toprule
& \acc in \% $(\uparrow)$ & \nll $(\downarrow)$ & \ece $(\downarrow)$& Deterministic parameters in \%\\
\midrule
DWF-PSNN& 78.02 $\pm$ 0.61 & 0.508 $\pm$ 0.010 & 0.022 $\pm$ 0.010 & 60.90 $\pm$ 2.61\\
ALS-BNN& \textbf{78.03 $\pm$ 0.61} & 0.502 $\pm$ 0.009 & 0.019 $\pm$ 0.006 & 0.00\\
FLS-BNN& \textbf{78.03 $\pm$ 0.61} & 0.506 $\pm$ 0.009 & 0.015 $\pm$ 0.005& 81.11\\
RAND-PSNN& 78.02 $\pm$ 0.61 & \textbf{0.500 $\pm$ 0.009} & 0.023 $\pm$ 0.006 & 60.68 $\pm$ 2.61\\
SNI-PSNN& 78.02 $\pm$ 0.61 & 0.527 $\pm$ 0.007 & \textbf{0.014 $\pm$ 0.013} & 90.00\\
\bottomrule
\end{tabular}
\end{table}

\begin{table}[H]
\centering
\caption{Results on \texttt{CONCRETE} ($\mathrm{mean}\pm\mathrm{std}$).}
\label{tab:reg_concrete}
\small
\setlength{\tabcolsep}{3pt}
\begin{tabular}{ccccc}
\toprule
& \rmse $(\downarrow)$ & \nll $(\downarrow)$ & \ece $(\downarrow)$& Deterministic parameters in \%\\
\midrule
DWF-PSNN& \textbf{0.298 $\pm$ 0.027} & \textbf{0.214 $\pm$ 0.073} & 0.107 $\pm$ 0.066 & 48.69 $\pm$ 1.27\\
ALS-BNN& 0.300 $\pm$ 0.032& 0.283 $\pm$ 0.068 & 0.134 $\pm$ 0.062 & 0.00 \\
FLS-BNN& 0.335 $\pm$ 0.030 & 0.286 $\pm$ 0.085 & 0.085 $\pm$ 0.062 & 91.89\\
RAND-PSNN& 0.331 $\pm$ 0.034 & 0.322 $\pm$ 0.113 & \textbf{0.075 $\pm$ 0.073} & 48.69 $\pm$ 1.27\\
SNI-PSNN& 0.332 $\pm$ 0.039 & 0.326 $\pm$ 0.133 & 0.097 $\pm$ 0.088 & 90.00\\
\bottomrule
\end{tabular}
\end{table}

\begin{table}[H]
\centering
\caption{Results on \texttt{WINE} ($\mathrm{mean}\pm\mathrm{std}$).}
\label{tab:reg_wine}
\small
\setlength{\tabcolsep}{3pt}
\begin{tabular}{ccccc}
\toprule
& \rmse $(\downarrow)$ & \nll $(\downarrow)$ & \ece $(\downarrow)$& Deterministic parameters in \%\\
\midrule
DWF-PSNN& 0.771 $\pm$ 0.023 & 2.210 $\pm$ 0.145 & 0.127 $\pm$ 0.015 & 47.12 $\pm$ 2.93 \\
ALS-BNN& \textbf{0.723 $\pm$ 0.031} & \textbf{1.049 $\pm$ 0.065} & \textbf{0.017 $\pm$ 0.015} & 0.00\\
FLS-BNN& 0.856 $\pm$ 0.024 & 3.395 $\pm$ 0.202 & 0.207 $\pm$ 0.017 & 90.27 \\
RAND-PSNN& 0.841 $\pm$ 0.035 & 3.314 $\pm$ 0.294 & 0.149 $\pm$ 0.050 & 47.07 $\pm$ 2.93\\
SNI-PSNN& 0.844 $\pm$ 0.036 & 3.338 $\pm$ 0.304& 0.135 $\pm$ 0.029& 90.00\\
\bottomrule
\end{tabular}
\end{table}

\begin{table}[H]
\centering
\caption{Results on \texttt{BIKESHARING} ($\mathrm{mean}\pm\mathrm{std}$).}
\label{tab:reg_bikesharing}
\small
\setlength{\tabcolsep}{3pt}
\begin{tabular}{ccccc}
\toprule
& \rmse $(\downarrow)$ & \nll $(\downarrow)$ & \ece $(\downarrow)$& Deterministic parameters in \%\\
\midrule
DWF-PSNN& 0.105 $\pm$ 0.034 & -0.159 $\pm$ 0.034 & 0.585 $\pm$ 0.196& 47.81 $\pm$ 2.95\\
ALS-BNN& 0.155 $\pm$ 0.037 & 0.049 $\pm$ 0.038 & \textbf{0.432 $\pm$ 0.084} & 0.00\\
FLS-BNN& \textbf{0.101 $\pm$ 0.040} & -0.164 $\pm$ 0.041 & 0.614 $\pm$ 0.213 & 88.70\\
RAND-PSNN& \textbf{0.101 $\pm$0.037} & \textbf{-0.174 $\pm$ 0.050} & 0.600 $\pm$ 0.191 & 47.75 $\pm$ 2.95\\
SNI-PSNN& 0.104 $\pm$ 0.033 &  -0.173 $\pm$ 0.033 & 0.674 $\pm$ 0.180 & 90.00\\
\bottomrule
\end{tabular}
\end{table}

\begin{table}[H]
\centering
\caption{Results on \texttt{BOSTON} ($\mathrm{mean}\pm\mathrm{std}$).}
\label{tab:reg_boston}
\small
\setlength{\tabcolsep}{3pt}
\begin{tabular}{ccccc}
\toprule
& \rmse $(\downarrow)$ & \nll $(\downarrow)$ & \ece $(\downarrow)$& Deterministic parameters in \%\\
\midrule
DWF-PSNN& \textbf{0.386 $\pm$ 0.113} & 0.482 $\pm$ 0.331 & 0.209 $\pm$ 0.145& 48.03 $\pm$ 2.53\\
ALS-BNN& 0.451 $\pm$ 0.107 & \textbf{0.470 $\pm$ 0.170}& 0.226 $\pm$ 0.166 & 0.00\\
FLS-BNN& 0.392 $\pm$ 0.116 & 0.513 $\pm$ 0.337 & 0.213 $\pm$ 0.172 & 87.93\\
RAND-PSNN& 0.398 $\pm$ 0.102 & 0.608 $\pm$ 0.385 & \textbf{0.196 $\pm$ 0.193 }& 47.97 $\pm$ 2.53\\
SNI-PSNN& 0.408 $\pm$ 0.106 & 0.652 $\pm$ 0.425 & 0.310 $\pm$ 0.305 & 90.00\\
\bottomrule
\end{tabular}
\end{table}

\section{The learned split}
\label{App:LearnedSplit}
\Cref{fig:learned_split} and \Cref{tab:learned_split} summarize the split learned by the prior optimization. The split is stable across splits of the data, with standard deviations of at most $0.04$ in the share of deterministic weights. The first and the output layer are largely deterministic, while the middle layer keeps about half of its weights stochastic. Since the middle layer holds the large majority of the parameters, it dominates the overall share of deterministic parameters reported in \Cref{App:AddRes}. The classification datasets are sparsified slightly more than the regression datasets.

The middle layer is sparsified least, although its regularization weight exceeds that of the first layer. This does not contradict Remark~\ref{rmk:prior_opt}, which assumes comparable gradient bounds $G_l$ across layers. The middle layer has a larger fan-in, and the gradient of the objective with respect to its prior scales is therefore unlikely to be bounded by the same constant as in the other layers.

In every dataset and split, between $7$ and $27$ of the $100$ first-layer biases remain stochastic. The coverage condition $U(S) \geq d_L$ of Proposition~\ref{prop:certificate} therefore holds with a wide margin. Before repair, the condition $|T| \geq d_0$ holds on \texttt{BANKNOTE}, \texttt{HTRU2} and in nine of ten splits of \texttt{CONCRETE}. On the other datasets the repair changes on average between 5.5 and 28.1 parameters per split, with the largest number of changes on \texttt{CREDIT}, which has the most input features (cf. \Cref{tab:certificate_uci}). Whether the output bias is deterministic varies between datasets, from none of the splits on \texttt{BANKNOTE} to all of them on \texttt{CREDIT} and \texttt{WINE}. The certificate places no condition on the output bias, so this variation does not affect the expressiveness of the masked network. All values refer to the learned masks before repair.
\begin{table}[ht]
\centering
\caption{Learned split of DWF-PSNN. The first columns give the share of deterministic parameters per layer, with mean and standard deviation over the ten splits. The next column gives the range over splits of the number of stochastic first-layer biases out of 100. The last column gives the number of splits in which the output bias is deterministic.}
\label{tab:learned_split}
\small
\setlength{\tabcolsep}{4pt}
\begin{tabular}{lccccc}
\toprule
& \multicolumn{3}{c}{Deterministic parameters} & Stochastic & Deterministic \\
\cmidrule(lr){2-4}
Dataset & Layer 1 & Layer 2 & Output layer & first-layer biases & output bias \\
\midrule
\texttt{ONLINESHOPPERS} & $0.87 \pm 0.01$ & $0.56 \pm 0.04$ & $0.86 \pm 0.03$ & 9--17 & 6/10 \\
\texttt{CREDIT} & $0.86 \pm 0.01$ & $0.54 \pm 0.03$ & $0.89 \pm 0.02$ & 7--20 & 10/10 \\
\texttt{BANKNOTE} & $0.87 \pm 0.02$ & $0.53 \pm 0.03$ & $0.84 \pm 0.02$ & 8--18 & 0/10 \\
\texttt{HTRU2} & $0.87 \pm 0.01$ & $0.54 \pm 0.03$ & $0.89 \pm 0.02$ & 10--18 & 7/10 \\
\midrule
\texttt{CONCRETE} & $0.80 \pm 0.01$ & $0.46 \pm 0.01$ & $0.78 \pm 0.04$ & 11--22 & 4/10 \\
\texttt{BOSTON} & $0.81 \pm 0.01$ & $0.43 \pm 0.03$ & $0.78 \pm 0.03$ & 12--24 & 5/10 \\
\texttt{BIKESHARING} & $0.82 \pm 0.01$ & $0.43 \pm 0.03$ & $0.80 \pm 0.04$ & 13--25 & 4/10 \\
\texttt{WINE} & $0.80 \pm 0.01$ & $0.43 \pm 0.03$ & $0.86 \pm 0.03$ & 18--27 & 10/10 \\
\bottomrule
\end{tabular}
\end{table}

\section{The conditions of the certificate}
\label{app:ucda}

\begin{table}[ht]
\centering
\small
\caption{The violated conditions are the ones in Proposition~\ref{prop:certificate}. Values are the mean Wasserstein-1
distance to the true conditional over nine held out inputs, with the
standard error over five data seeds. The sampling floor is the distance
between two independent samples from the true conditional at this sample
size, so no model can go below it, and the point predictor is what a
model reaches by matching the conditional median without representing any
spread. The last column repeats the study with learnable prior scales.}
\label{tab:ucda}
\begin{tabular}{llccc}
\toprule
 & & & \multicolumn{2}{c}{$W_1$} \\
\cmidrule(lr){4-5}
variant & condition violated & $n_{\mathrm{stoch}}$
        & fixed & learned \\
\midrule
sampling floor  & --- & --- & 0.020 & 0.020 \\
point predictor & --- & 0   & 0.681 & 0.681 \\
\midrule
noise input, $d_L$   & reference, none & 0 & $0.053 \pm 0.004$ & $0.053 \pm 0.004$ \\
first layer, $d_L$   & none            & 1 & $0.074 \pm 0.002$ & $0.075 \pm 0.002$ \\
\midrule
first layer, $d_L$, weight & vanishes at $x=0$ & 1    & $0.284 \pm 0.002$ & $0.320 \pm 0.010$ \\
output layer, $d_L$        & bias coverage     & 1    & $0.568 \pm 0.002$ & $0.588 \pm 0.003$ \\
output layer, all        & bias coverage     & 64   & $0.495 \pm 0.004$ & $0.465 \pm 0.005$ \\
first layer, $d_L$, linear & activation        & 1    & $0.935 \pm 0.002$ & $0.935 \pm 0.002$ \\
full                     & deterministic units              & 4353 & $5.411 \pm 0.008$ & $0.111 \pm 0.005$ \\
\midrule
deterministic & no randomness & 0 & $0.686 \pm 0.002$ & $0.686 \pm 0.002$ \\
\bottomrule
\end{tabular}
\end{table}

\begin{figure}[ht]
\centering
\includegraphics[width=0.5\textwidth]{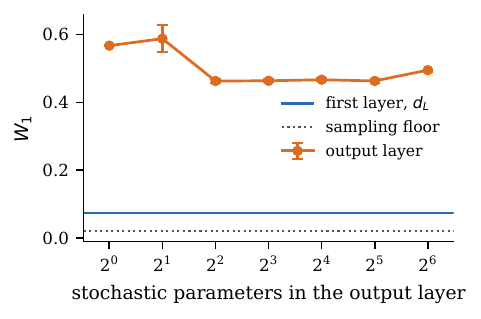}
\caption{Distance to the true conditional against the number of
randomized parameters in the output layer, with the certified
configuration and the sampling floor as horizontal lines. Mean over five
data seeds with the standard error.}
\label{fig:ucda-count}
\end{figure}

\paragraph{Setup.}
The target has a bimodal conditional distribution with an input dependent mixture
weight, so that any Gaussian conditional is excluded by construction.
Since the certificate is a statement about representation, posterior
inference would confound capacity with the quality of the sampler. Every
variant is therefore fitted directly as a conditional generator by
minimizing the energy score, which is strictly proper, and is scored by
the Wasserstein-1 distance to the true conditional. Within each data seed
we keep the best of five restarts, selected on the validation energy
score and never on the reported distance, because the statement is an
existence statement about parameters. All variants share one
architecture with two hidden layers of width sixty-four and differ only
in which parameters are randomized.
In the column \emph{fixed}, every stochastic parameter has prior scale $\sigma = 1$. In the column \emph{learned}, the scales are parameterized by a softplus and fitted jointly with the means, without a sparsity penalty.

\paragraph{The reference acts as a gate.}
In the reference the randomness enters as an additional input and every
weight is deterministic, so the map is unconstrained. It reaches $0.053$
against a sampling floor of $0.020$, which shows that the fitting
procedure is not the binding constraint and that the remaining
differences can be read as consequences of the condition each variant
violates.

\paragraph{Why biases and not weights.}
The variant that randomizes a weight of the first layer rather than a
bias satisfies the count and the position, and still reaches only
$0.284$. A stochastic input weight contributes
$\sigma_{j,i}\varepsilon_{j,i}x_i$, which vanishes at $\mathbf x=\mathbf 0$, so the scale
$s_j$ in the proof of Proposition~\ref{prop:certificate} is not bounded
away from zero and the predictive distribution degenerates there. This is the
reason Definition~\ref{def:bias_coverage} counts biases.

\paragraph{Fixed and learnable prior scales.}
Only the fully stochastic row depends on this choice, and it changes by a
factor of fifty. With learnable scales that variant contains the
certified one as a special case, since it can drive all but $d_L$
standard deviations to zero, and the optimizer finds such a
configuration. The condition $|T| \geq d_0$ therefore constrains a fixed mask with
fixed scales, which is the setting of our two stage method, where the
prior scales are determined in the first stage and held during sampling.
With fixed scales the variant is not merely worse but wildly
overdispersed, at eight times the distance of a point predictor, which is
why its panel in \Cref{fig:ucda_toy} appears flat.

\paragraph{Count does not substitute for position.}
Figure~\ref{fig:ucda-count} increases the number of randomized parameters
in the output layer from one to sixty-four. The distance is flat between
$0.46$ and $0.59$ throughout and never approaches the certified
configuration at $0.074$, so a factor of sixty-four in the count does not
compensate for a position the certificate rejects.

\begin{figure}[H]
\centering
\includegraphics[width=\textwidth]{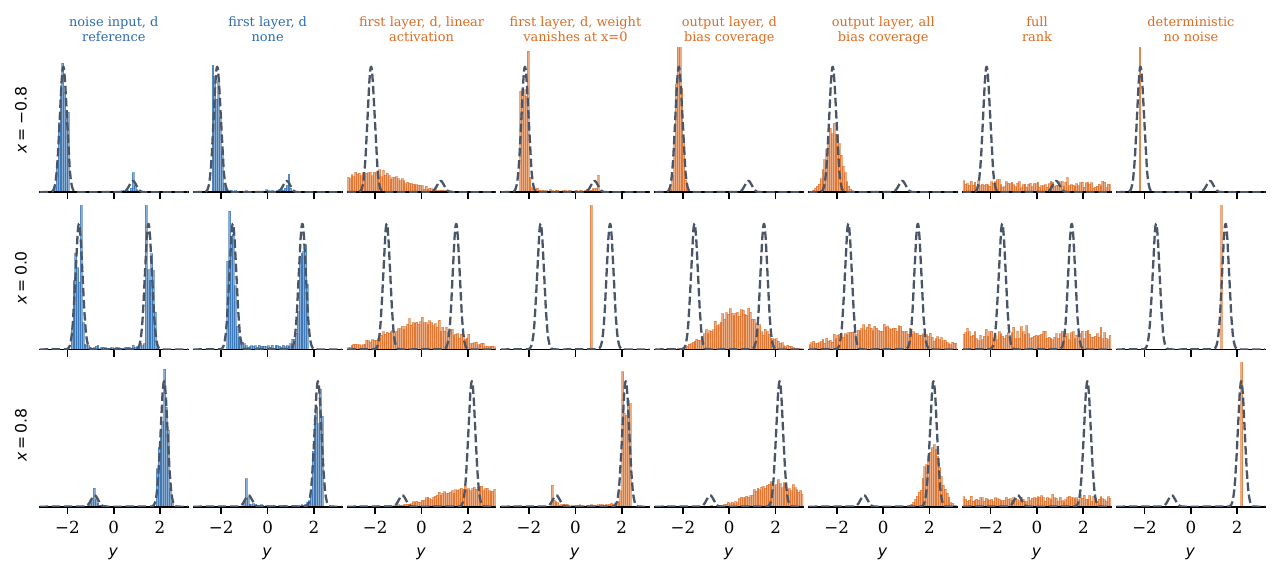}
\caption{All variants of Table~\ref{tab:ucda} at three inputs, with fixed
prior scales. In the panel titles $d$ denotes the output dimension $d_L$, rank denotes the condition $|T|\geq d_0$ on the fully deterministic units and no noise corresponds to no randomness in \Cref{tab:ucda}. Grey dashed is the true conditional, blue marks the
configurations the certificate admits. The vertical range is fixed by the
target density, so mass outside it appears flat.}
\label{fig:ucda-appendix}
\end{figure}
\Cref{fig:toy_u_hist} shows the prior scales behind \Cref{fig:toy_cutoff}. For $\lambda = 10^{-5}$ most factors are driven many orders of magnitude below the cutoff, which shows that the penalty of \Cref{sec:pri_opt} pushes pruned scales towards zero. For $\lambda = 3 \cdot 10^{-6}$ the scales form a continuum across the cutoff.
\begin{figure}[H]
\centering
\includegraphics[width=0.6\textwidth]{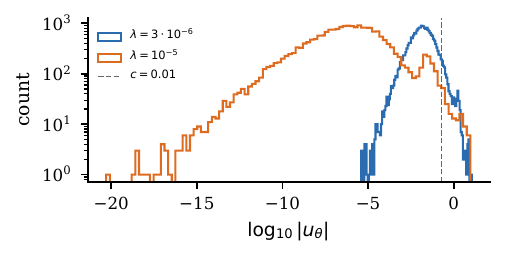}
\caption{Histogram of $\log_{10}|u_\theta|$ on the bimodal target over five data seeds, with the cutoff $c = 0.01$ on the prior scale, which corresponds to $|u| \approx  0.2$, as a dashed line. For larger $\lambda$ the pruned products $u_\theta$ collapse towards zero.}
\label{fig:toy_u_hist}
\end{figure}
\section{Experimental Details}
\label{App:experimental_details}
We use a \mlp with two hidden layers of 100 units each and a $\tanh$ activation function, where not stated differently. All regression and classification experiments were run on an Nvidia A40 GPU, and each model finished in less than two hours. The runs for the role of $\gamma$ and the \ucda experiment have also been performed on an Nvidia A40 GPU and finished within hours. All other experiments were performed on a MacBook Pro 2021 with Apple M1 Pro chip and took between a few minutes and a few hours. Unless stated otherwise, we use the random seed 123 for initialization and sampling, and all methods share the same ten data splits. We build on the code base provided by \citep{tran2022needgoodfunctionalprior}. The repository does not state a license. The accompanying paper is published under CC-BY 4.0. We have not performed an extensive hyperparameter search, but in most cases tried out a few configurations and picked the best one.
Additionally, we use a two-phase learning-rate schedule. Let $T$ denote the total
number of optimization steps and let
$T_0 = \lfloor T/2 \rfloor$. For optimizer $j$, the learning rate at
step $t \in \{0,\ldots,T-1\}$ is given by
\begin{equation}
    \eta_t^{(j)}
    =
    \eta_0^{(j)}
    \begin{cases}
        1,
        & t < T_0, \\[4pt]
        (t - T_0 + 1)^{-\alpha_j},
        & t \geq T_0,
    \end{cases}
    \label{eq:lr_schedule}
\end{equation}
where $\alpha_{\mathrm{SGHMC}} = 0.2$ and
$\alpha_{\mathrm{Adam}} = 0.4$.

The learned masks satisfy the coverage condition in all runs. The condition $|T| \geq d_0$ fails on the datasets with many input features, where fewer than $d_0$ first-layer units are fully deterministic. In these runs we apply the minimal repair of \Cref{cor:MinRep}, which changes between 1 and 34 parameters per run (\Cref{tab:certificate_uci}).
\subsection{Regression experiments}
For regression, \ece denotes the calibration error of
\citet{kuleshov2018accurate}. We summarize the predictive distribution of
test point $i$ by a Gaussian with the predictive mean and variance and
evaluate its cumulative distribution function $F_i$ at the target $y_i$.
For the confidence levels $p_j = j/10$ with $j = 1, \dots, 10$, let
$\hat p_j$ be the share of test points with $F_i(y_i) < p_j$. We report
$\sum_{j=1}^{10} (p_j - \hat p_j)^2$. This is a sum over the levels and
not an average, so its scale is not comparable to that of the
classification \ece. Additionally, we standardize the targets. 
The target \gp has a constant zero mean function, a variance of 1, a length-scale of $\ell = \sqrt{2 \cdot d_0}$, and a noise variance of 0.1. On top of that, we have the following hyperpriors:
\[
\ell_i \sim \operatorname{LogNormal}\big(\log(\ell),\, 1\big), \quad i = 1, \dots, d_0
\]
\[
\alpha^2 \sim \operatorname{LogNormal}(0.1,\, 1)
\]
We set D=3 and the cutoff as 0.01.
The configuration for the prior optimization are:
\begin{table}[H]
\centering
\begin{tabular}{ll}
\toprule
\textbf{Hyperparameter} & \textbf{Value} \\
\midrule
Batch size & 64 \\
Number of iterations & 10000 \\
Regularization ($\lambda$) & (9, 12, 15) / D \\
Number of function samples & 100 \\
Learning rate & 0.01 \\
Kernel & Polynomial of degree 2 \\
\bottomrule
\end{tabular}
\caption{Hyperparameters used for prior optimization}
\label{tab:prior_hyperparameters_1}
\end{table}

For sampling we have used the following hyperparameters:
\begin{table}[H]
\centering
\begin{tabular}{ll}
\toprule
\textbf{Hyperparameter} & \textbf{Value} \\
\midrule
Batch size & 64 \\
Number of samples per chain & 30 \\
Discarded samples before sampling  & 10 \\
Thinning interval & 200 \\
Number of chains & 4 \\
Momentum decay & 0.02 \\
Burn in steps& 3000 \\
\sghmc learning rate (h) & 0.03 \\
\texttt{Adam} learning rate ($\eta$) & 0.003\\
\bottomrule
\end{tabular}
\caption{Hyperparameters used for all models but SNI-PSNN}
\label{tab:sampling_hyperparameters_1}
\end{table}
We have set $\gamma = 2.5$.
\begin{table}[H]
\centering
\begin{tabular}{ll}
\toprule
\textbf{Hyperparameter} & \textbf{Value} \\
\midrule
Batch size & 64 \\
Prior precision & 1.0 \\
Stochastic fraction & 0.1 \\
Noise variance & 0.1 \\
Epochs & 1000 \\
Learning rate & 0.01\\
\bottomrule
\end{tabular}
\caption{Hyperparameters used for SNI-PSNN}
\label{tab:sampling_hyperparameters_sni_psnn_1}
\end{table}

\subsection{Classification experiment}
The \gp is set up as in the regression experiment. For the prior optimization we also use the same hyperparameters, but with a batch size of 128. For the sampling we use the following hyperparameters:
\begin{table}[H]
\centering
\begin{tabular}{ll}
\toprule
\textbf{Hyperparameter} & \textbf{Value} \\
\midrule
Batch size & 64 \\
Number of samples per chain & 30 \\
Discarded samples before sampling  & 10 \\
Thinning interval & 200 \\
Number of chains & 4 \\
Momentum decay & 0.02 \\
Burn in steps & 3000 \\
\sghmc learning rate (h) & 0.01 \\
\texttt{Adam} learning rate ($\eta$) & 0.001 \\
\bottomrule
\end{tabular}
\caption{Hyperparameters used for all models but SNI-PSNN}
\label{tab:sampling_hyperparameters_cla_1}
\end{table}
Again we have set $\gamma = 2.5$.
\begin{table}[H]
\centering
\begin{tabular}{ll}
\toprule
\textbf{Hyperparameter} & \textbf{Value} \\
\midrule
Batch size & 64 \\
Prior precision & 1.0 \\
Stochastic fraction & 0.1 \\
Noise variance & 0.1 \\
Epochs & 1000 \\
Learning rate & 0.01\\
\bottomrule
\end{tabular}
\caption{Hyperparameters used for SNI-PSNN}
\label{tab:sampling_hyperparameters_sni_psnn_2}
\end{table}
\subsection{Prior optimization}
The hyperpriors are the same, but the  hyperparameters are as follows:
\begin{table}[H]
\centering
\begin{tabular}{ll}
\toprule
\textbf{Hyperparameter} & \textbf{Value} \\
\midrule
Batch size & 128 \\
Number of iterations & 15000 \\
Regularization ($\lambda$) & (1.2, 2.4, 4.8) / D \\
Number of function samples & 100 \\
Learning rate & 0.03 \\
Kernel & Polynomial of degree 2 \\
\bottomrule
\end{tabular}
\caption{Hyperparameters used for prior optimization}
\label{tab:prior_hyperparameters_2}
\end{table}
The sampling configurations are:
\begin{table}[H]
\centering
\begin{tabular}{ll}
\toprule
\textbf{Hyperparameter} & \textbf{Value} \\
\midrule
Batch size & 64 \\
Number of samples per chain & 60 \\
Discarded samples before sampling  & 20 \\
Thinning interval & 100 \\
Number of chains & 4 \\
Momentum decay & 0.01 \\
Burn in steps & 10000 \\
\sghmc learning rate (h) & 0.03 \\
\texttt{Adam} learning rate ($\eta$) & 0.003 \\
\bottomrule
\end{tabular}
\caption{Sampling hyperparameters}
\label{tab:sampling_hyperparameters_prior_opt}
\end{table}
Additionally, we have set $\gamma = 25$.
\subsection{Role of $\gamma$}
We have used the same hyperparameters for the prior optimization as in the regression experiments. For sampling we have used:
\begin{table}[H]
\centering
\begin{tabular}{ll}
\toprule
\textbf{Hyperparameter} & \textbf{Value} \\
\midrule
Batch size & 64 \\
Number of samples per chain & 100 \\
Discarded samples before sampling  & 10 \\
Number of burn-in steps & 10000\\
Thinning interval & 20 \\
Number of chains & 4 \\
Momentum decay & 0.02 \\
\sghmc sampler rate $(h)$& 0.03\\
\texttt{Adam} learning rate $(\eta)$& 0.003\\
\bottomrule
\end{tabular}
\caption{Hyperparameters used for sampling}
\label{tab:sampling_hyperparameters_cla_2}
\end{table}
For (ii) we use $\gamma = 0.03$ and $\gamma = 2.5$, and for (iii) $\gamma = 10^{13}$.
\subsection{Learned and random splits on the bimodal target}
\label{app:bimodal_target}
We use the architecture, target and energy-score fitting of \Cref{app:ucda}. In a first stage, the prior scales are parameterized as $\sigma_\theta = H(u_\theta)$ with $D = 3$ factors and fitted jointly with the means under the penalty of \Cref{sec:pri_opt}, with $\lambda \in \{3 \cdot 10^{-7}, 7 \cdot 10^{-7}, 1.5 \cdot 10^{-6}, 3 \cdot 10^{-6}, 6 \cdot 10^{-6}, 10^{-5}, 1.5 \cdot 10^{-5}\}$ and layer weights in the ratio $3:4:5$. Scales below $c = 0.01$ are set to zero. Unlike in the \texttt{UCI} experiments, we apply this threshold only after the first stage, so that all scales receive gradients during the optimization. In a second stage, the means are refitted with the mask and the scales held fixed. Random masks permute the thresholded scales either within each layer and parameter type or across all parameters. We use five data seeds, three random masks of each kind per seed and 12{,}000 steps per fit. \subsection{Sensitivity to the cutoff}
We use the setting of \Cref{app:bimodal_target} with $\lambda \in \{3 \cdot 10^{-6}, 10^{-5}\}$ and five data seeds. After the first stage we threshold the learned scales at $c \in \{0.003, 0.01, 0.03, 0.1\}$ and refit the means with the mask and the scales held fixed, using 12{,}000 steps per fit.
\subsection{Certificate of random masks}
For each dataset we draw 2{,}000 random first-layer masks with the learned shares of deterministic weights and biases and check both conditions of Proposition~\ref{prop:certificate}. At the learned budget, random masks satisfy the certificate on \texttt{BANKNOTE} and \texttt{HTRU2}, in 97\% of the draws on \texttt{CONCRETE}, and in at most 9\% of the draws on the other datasets.
\end{document}